\documentclass{article}

\PassOptionsToPackage{numbers}{natbib}
\usepackage[preprint]{neurips_2019}

\usepackage[utf8]{inputenc} 
\usepackage[T1]{fontenc}    
\usepackage[english]{babel}
\usepackage{url}            
\usepackage{booktabs}       
\usepackage{amsfonts}       
\usepackage{nicefrac}       
\usepackage{microtype}      
\usepackage{multirow}
\usepackage{tipa}
\usepackage{amsfonts}       
\usepackage{amssymb}
\usepackage{algorithmic}
\usepackage{algorithm}
\usepackage{amsthm}
\usepackage{parallel}
\usepackage{wrapfig}
\usepackage{authblk}
\usepackage{graphicx, color}
\usepackage[moderate]{savetrees}

\renewcommand{\cite}[1]{\citep{#1}}
\newtheorem{theorem}{Theorem}[section]
\newtheorem*{theorem*}{Theorem}
\newtheorem{corollary}{Corollary}[theorem]
\newtheorem*{corollary*}{Corollary}
\newtheorem*{proposition*}{Proposition}
\newtheorem{lemma}[theorem]{Lemma}
\newtheorem{proposition}[theorem]{Proposition}
\newtheorem{remark}[theorem]{Remark}
\theoremstyle{definition}
\newtheorem{definition}{Definition}[section]

\title{Making AI Forget You:\\ Data Deletion in Machine Learning} 

\author[1]{Antonio A. Ginart}
\author[2]{Melody Y. Guan}
\author[2]{Gregory Valiant}
\author[3]{James Zou}

\affil[1]{Dept. of Electrical Engineering}
\affil[2]{Dept. of Computer Science}
\affil[3]{Dept. of Biomedial Data Science}
\affil[ ]{Stanford University, Palo Alto, CA 94305}
\affil[ ]{\texttt{\{tginart, mguan, valiant, jamesz\}@stanford.edu }}

\begin{document}

\maketitle
\vspace{-15pt}
\begin{abstract}
\vspace{-5pt}
 Intense recent discussions have focused on how to provide individuals with control over when their data can and cannot be used --- the EU’s Right To Be Forgotten regulation is an example of this effort. In this paper we initiate a framework studying what to do when it is no longer permissible to deploy models derivative from specific user data. In particular, we formulate the problem of efficiently deleting individual data points from trained machine learning models. For many standard ML models, the only way to completely remove an individual's data is to retrain the whole model from scratch on the remaining data, which is often not computationally practical.  We investigate algorithmic principles that enable efficient data deletion in ML. For the specific setting of $k$-means clustering, we propose two provably efficient deletion algorithms which achieve an average of over $100\times$ improvement in deletion efficiency across 6 datasets, while producing clusters of comparable statistical quality to a canonical $k$-means++ baseline.

\end{abstract}

\vspace{-5pt}
\section{Introduction}
\vspace{-5pt}
Recently, one of the authors received the redacted email below, informing us that an individual's data cannot be used any longer. The UK Biobank 
\cite{sudlow2015uk} is one of the most valuable collections of genetic and medical records with half a million participants. Thousands of machine learning classifiers are trained on this data, and thousands of papers have been published using this data.

\fbox{\begin{minipage}{38em}
\scriptsize
\texttt{EMAIL ---- UK BIOBANK ----}\\ 
\texttt{Subject: UK Biobank Application [REDACTED], Participant Withdrawal Notification [REDACTED]}\\
\\
\texttt{Dear Researcher,}\\

\texttt{As you are aware, participants are free to withdraw form the UK Biobank at any time and request that their data no longer be used.}\texttt{ Since our last review, some participants involved with Application [REDACTED] have requested that their data should longer be used.}
\end{minipage}
}

The email request from the UK Biobank illustrates a fundamental challenge the broad data science and policy community is grappling with: \emph{how should we provide individuals with flexible control over how corporations, governments, and researchers use their data?} Individuals could decide at any time that they do not wish for their personal data to be used for a particular purpose by a particular entity. This ability is sometimes legally enforced. For example, the European Union's General Data Protection Regulation (GDPR) and former Right to Be Forgotten \cite{gdpr,righttobeforgotten} both require that companies and organizations enable users to withdraw consent to their data at any time under certain circumstances. These regulations broadly affect international companies and technology platforms with EU customers and users. Legal scholars have pointed out that the continued use of AI systems directly trained on deleted data could be considered illegal under certain interpretations and ultimately concluded that: \emph{it may be impossible to fulfill the legal aims of the Right to be Forgotten in artificial intelligence environments} \cite{villaronga2018humans}. Furthermore, so-called \emph{model-inversion attacks} have demonstrated the capability of adversaries to extract user information from trained ML models \cite{veale2018algorithms}.

Concretely, we frame the problem of data deletion in machine learning as follows. Suppose a statistical model is trained on $n$ datapoints. For example, the model could be trained to perform disease diagnosis from data collected from $n$ patients. To \emph{delete} the data sampled from the $i$-th patient from our trained model, we would like to update it such that it becomes independent of sample $i$, and looks as if it had been trained on the remaining $n-1$ patients.  A naive approach to satisfy the requested deletion would be to retrain the model from scratch on the data from the remaining $n-1$ patients. For many applications, this is not a tractable solution -- the costs (in time, computation, and energy) for training many machine learning models can be quite high. Large scale algorithms can take weeks to train and consume large amounts of electricity and other resources. Hence, we posit that efficient data deletion is a fundamental data management operation for machine learning models and AI systems, just like in relational databases or other classical data structures.

Beyond supporting individual data rights, there are various other possible use cases in which efficient data deletion is desirable. To name a few examples, it could be used to speed-up leave-one-out-cross-validation \cite{abu2012learning}, support a user data marketplace \cite{schomm2013marketplaces,truong2012data}, or identify important or valuable datapoints within a model \cite{ghorbani2019data}.

Deletion efficiency for general learning algorithms has not been previously studied. While the desired output of a deletion operation on a \emph{deterministic} model is fairly obvious, we have yet to even define data deletion for stochastic learning algorithms. At present, there is only a handful of learning algorithms known to support fast data deletion operations, all of which are deterministic. Even so, there is no pre-existing notion of how engineers should think about the asymptotic \emph{deletion efficiency} of learning systems, nor understanding of the kinds of trade-offs such systems face. 

The key components of this paper include introducing deletion efficient learning, based on an intuitive and operational notion of what it means to (efficiently) delete data from a (possibly stochastic) statistical model. We pose data deletion as an online problem, from which a notion of optimal deletion efficiency emerges from a natural lower bound on amortized computation time. We do a case-study on deletion efficient learning using the simple, yet perennial, $k$-means clustering problem. We propose two deletion efficient algorithms that (in certain regimes) achieve optimal deletion efficiency. Empirically, on six datasets, our methods achieve an average of over $100\times$ speedup in amortized runtime with respect to the canonical Lloyd's algorithm seeded by $k$-means++ \cite{lloyd1982least,arthur2007k}. Simultaneously, our proposed deletion efficient algorithms perform comparably to the canonical algorithm on three different statistical metrics of clustering quality. Finally, we synthesize an algorithmic toolbox for designing  deletion efficient learning systems.

We summarize our work into three contributions:

\textbf{(1)} We formalize the problem and notion of efficient data deletion in the context of machine learning. 

\textbf{(2)} We propose two different deletion efficient solutions for $k$-means clustering that have theoretical guarantees and strong empirical results.

\textbf{(3)} From our theory and experiments, we synthesize four general engineering principles for designing deletion efficient learning systems.

\vspace{-5pt}
 \section{Related Works}
 \vspace{-5pt}

\paragraph{Deterministic Deletion Updates}
As mentioned in the introduction, efficient deletion operations are known for some canonical learning algorithms. They include linear models \cite{statisticalcomputation,regression,van1983matrix,tsai2014incremental,cao2015towards, schelteramnesia}, certain types of \emph{lazy learning} \cite{Webb2010,atkeson1997locally,Birattari:1999:LLM:340534.340673} techniques such as non-parametric Nadaraya-Watson kernel regressions \cite{nadaraya1964estimating} or nearest-neighbors methods \cite{coomans1982alternative,schelteramnesia}, recursive support vector machines \cite{cauwenberghs2001incremental,tsai2014incremental}, and co-occurrence based collaborative filtering \cite{schelteramnesia}. 

\paragraph{Data Deletion and Data Privacy}

Related ideas for protecting data in machine learning --- e.g. cryptography \cite{ohrimenko2016oblivious,bost2015machine,bonawitz2017practical,bogdanov2018implementation,nikolaenko2013privacy,erkin2012generating}, and differential privacy \cite{dwork2014algorithmic,chaudhuri2013near,chaudhuri2011differentially,papernot2016semi,abadi2016deep} --- do not lead to efficient data deletion, but rather attempt to make data private or non-identifiable. Algorithms that support efficient deletion do not have to be private, and algorithms that are private do not have to support efficient deletion. To see the difference between privacy and data deletion, note that every learning algorithm supports the naive data deletion operation of retraining from scratch. The algorithm is not required to satisfy any privacy guarantees. \emph{Even an operation that outputs the entire dataset in the clear could support data deletion, whereas such an operation is certainly not private}. In this sense, the challenge of data deletion only arises in the presence of computational limitations. Privacy, on the other hand, presents statistical challenges, even in the absence of any computational limitations. With that being said, data deletion has direct connections and consequences in data privacy and security, which we explore in more detail in Appendix A.

\vspace{-10pt}
\section{Problem Formulation}
\vspace{-5pt}

We proceed by describing our setting and defining the notion of \emph{data deletion} in the context of a machine learning algorithm and model. Our definition formalizes the intuitive goal that after a specified datapoint, $x$, is deleted, the resulting model is updated to be indistinguishable from a model that was trained from scratch on the dataset sans $x$. Once we have defined data deletion, we define a notion of \emph{deletion efficiency} in the context of an online setting. Finally, we conclude by synthesizing high-level principles for designing deletion efficient learning algorithms.

Throughout we denote dataset $D = \{x_1,\ldots,x_n\}$ as a set consisting of $n$ datapoints, with each datapoint $x_i \in \mathbf{R}^d$; for simplicity, we often represent $D$ as a $n \times d$ real-valued matrix as well. Let $A$ denote a (possibly randomized) algorithm that maps a dataset to a model in hypothesis space $\mathcal{H}$. We allow models to also include arbitrary metadata that is not necessarily used at inference time. Such metadata could include data structures or partial computations that can be leveraged to help with subsequent deletions. We also emphasize that algorithm $A$ operates on datasets of any size. Since $A$ is often stochastic, we can also treat $A$ as implicitly defining a conditional distribution over $\mathcal{H}$ given dataset $D$.

\begin{definition}{\textbf{Data Deletion Operation:}}
  We define a \textit{data deletion} operation for learning algorithm $A$, $R_A(D,A(D),i)$, which maps the dataset $D$, model $A(D)$, and index $i \in \{1,\ldots,n\}$ to some model in  $\mathcal{H}$. Such an operation is a data deletion operation if, for all $D$ and $i$, random variables $A(D_{-i})$ and $R_A(D,A(D),i)$ are equal in distribution, $A(D_{-i}) =_d R_A(D,A(D),i)$.
\end{definition}

Here we focus on exact data deletion: after deleting a training point from the model, the model should be as if this training point had never been seen in the first place. The above definition can naturally be relaxed to approximate data deletion by requiring a bound on the distance (or divergence) between distributions of $A(D_{-i})$ and $R_A(D,A(D),i)$. Refer to Appendix A for more details on approximate data deletion, especially in connection to differential privacy. We defer a full discussion of this to future work.

\paragraph{A Computational Challenge} Every learning algorithm, $A$, supports a trivial data deletion operation corresponding to simply retraining on the new dataset after the specified datapoint has been removed --- namely running algorithm $A$ on the dataset $D_{-i}$.  Because of this, the challenge of data deletion is computational: \textbf{1)} Can we design a learning algorithm $A$, and supporting data structures, so as to allow for a computationally efficient data deletion operation? \textbf{2)} For what algorithms $A$ is there a data deletion operation that runs in time sublinear in the size of the dataset, or at least sublinear in the time it takes to compute the original model, $A(D)$?  \textbf{3)} How do restrictions on the memory-footprint of the metadata contained in $A(D)$ impact the efficiency of 
data deletion algorithms?

\paragraph{Data Deletion as an Online Problem}
One convenient way of concretely formulating the computational challenge of data deletion is via the lens of online algorithms~\cite{bottou1998online}.  Given a dataset of $n$ datapoints, a specific training algorithm $A$, and its corresponding deletion operation $R_A$, one can consider a stream of $m \le n$ distinct indices, $i_1,i_2,\ldots,i_m \in \{1,\ldots,n\}$, corresponding to the sequence of datapoints to be deleted.  The online task then is to design a data deletion operation that is given the indices $\{i_j\}$ one at a time, and must output $A(D_{-\{i_1,\ldots,i_j\}})$ upon being given index $i_j$.  As in the extensive body of work on online algorithms, the goal is to minimize the amortized computation time. The amortized runtime in the proposed online deletion setting is a natural and meaningful way to measure deletion efficiency. A formal definition of our proposed online problem setting can be found in Appendix A.

In online data deletion, a simple lower bound on amortized runtime emerges. All (sequential) learning algorithms $A$ run in time $\Omega(n)$ under the natural assumption that $A$ must process each datapoint at least once. Furthermore, in the best case, $A$ comes with a constant time deletion operation (or a deletion oracle).

\begin{remark}
In the online setting, for $n$ datapoints and $m$ deletion requests we establish an asymptotic lower bound of $\Omega(\frac{n}{m})$  for the amortized computation time of any (sequential) learning algorithm.
\end{remark}

We refer to an algorithm achieving this lower bound as \emph{deletion efficient}. Obtaining tight upper and lower bounds is an open question for many basic learning paradigms including ridge regression, decision tree models, and settings where $A$ corresponds to the solution to a stochastic optimization problem. In this paper, we do a case study on $k$-means clustering, showing that we can achieve deletion efficiency without sacrificing statistical performance.

 \subsection{General Principles for Deletion Efficient Machine Learning Systems}
 
 We identify four design principles which we envision as the pillars of deletion efficient learning algorithms.
 
 \paragraph{Linearity} Use of linear computation allows for simple post-processing to undo the influence of a single datapoint on a set of parameters. Generally speaking, the Sherman-Morrison-Woodbury matrix identity and matrix factorization techniques can be used to derive fast and explicit formulas for updating linear models \cite{statisticalcomputation,regression,van1983matrix,higham2002accuracy}. For example, in the case of linear least squares regressions, QR factorization can be used to delete datapoints from learned weights in time $O(d^2)$ \cite{hammarling2008updating,zeb2017updating}. Linearity should be most effective in domains in which randomized \cite{rahimi2008random}, reservoir \cite{yin2012self,schrauwen2007overview},  domain-specific \cite{lowe1999object}, or pre-trained feature spaces elucidate linear relationships in the data.
 
\paragraph{Laziness} Lazy learning methods delay computation until inference time \cite{Webb2010,Birattari:1999:LLM:340534.340673,atkeson1997locally}, resulting in trivial deletions. One of the simplest examples of lazy learning is $k$-nearest neighbors \cite{friedman2001elements,altman1992introduction,schelteramnesia}, where deleting a point from the dataset at deletion time directly translates to an updated model at inference time. There is a natural affinity between lazy learning and non-parametric techniques \cite{nadaraya1964estimating,bontempi2001local}. Although we did not make use of laziness for unsupervised learning in this work, pre-existing literature on kernel density estimation for clustering would be a natural starting place \cite{hinneburg2007denclue}. Laziness should be most effective in regimes when there are fewer constraints on inference time and model memory than training time or deletion time. In some sense, laziness can be interpreted as shifting computation from training to inference. As a side effect, deletion can be immensely simplified.

\paragraph{Modularity} In the context of deletion efficient learning, modularity is the restriction of dependence of computation state or model parameters to specific partitions of the dataset. Under such a modularization, we can isolate specific modules of data processing that need to be recomputed in order to account for deletions to the dataset. Our notion of modularity is conceptually similar to its use in software design \cite{berman1993optimization} and distributed computing \cite{peleg2000distributed}. In DC-$k$-means, we leverage modularity by managing the dependence between computation and data via the divide-and-conquer tree. Modularity should be most effective in regimes for which the dimension of the data is small compared to the dataset size, allowing for partitions of the dataset to capture the important structure and features.
         
\paragraph{Quantization} Many models come with a sense of continuity from dataset space to model space --- small changes to the dataset should result in small changes to the (distribution over the) model. In statistical and computational learning theory, this idea is known to as \emph{stability} \cite{mukherjee2006learning,kearns1999algorithmic,kutin2002almost,devroye1979distribution,shalev2010learnability,poggio2004general}. We can leverage stability by quantizing the mapping from datasets to models (either explicitly or implicitly). Then, for a small number of deletions, such a quantized model is unlikely to change. If this can be efficiently verified at deletion time, then it can be used for fast average-case deletions. Quantization is most effective in regimes for which the number of parameters is small compared to the dataset size.

\vspace{-5pt}
\section{Deletion Efficient Clustering}
\vspace{-5pt}
Data deletion is a general challenge for machine learning. Due to its simplicity we focus on $k$-means clustering as a case study. Clustering is a widely used ML application, including on the UK Biobank (for example as in \cite{galinsky2016population}). We propose two algorithms for deletion efficient $k$-means clustering. In the context of $k$-means, we treat the output centroids as the model from which we are interested in deleting datapoints. We summarize our proposed algorithms and state theoretical runtime complexity and statistical performance guarantees. Please refer to \cite{friedman2001elements} for background concerning $k$-means clustering.

\subsection{Quantized $k$-Means}

We propose a quantized variant of Lloyd's algorithm as a deletion efficient solution to $k$-means clustering, called Q-$k$-means.  By quantizing the centroids at each iteration, we show that the algorithm's centroids are constant with respect to deletions with high probability. Under this notion of quantized stability, we can support efficient deletion, since most deletions can be resolved without re-computing the centroids from scratch. Our proposed algorithm is distinct from other quantized versions of $k$-means \cite{schellekens2018quantized}, which quantize the data to minimize memory or communication costs. We present an abridged version of the algorithm here (Algorithm 1). Detailed pseudo-code for Q-$k$-means and its deletion operation may be found in Appendix B.

Q-$k$-means follows the iterative protocol as does the canonical Lloyd's algorithm (and makes use of the $k$-means++ initialization). There are four key differences from Lloyd's algorithm. First and foremost, the centroids are quantized in each iteration before updating the partition. The quantization maps each point to the nearest vertex of a uniform $\epsilon$-lattice \cite{gray1998quantization}. To de-bias the quantization, we apply a random phase shift to the lattice. The particulars of the quantization scheme are discussed in Appendix B. Second, at various steps throughout the computation, we \emph{memoize} the optimization state into the model's metadata for use at deletion time (incurring an additional $O(ktd)$ memory cost). Third, we introduce a balance correction step, which compensates for $\gamma$-imbalanced clusters by averaging current centroids with a momentum term based on the previous centroids. Explicitly, for some $\gamma \in (0,1)$, we  consider any partition $\pi_\kappa$ to be $\gamma$-imbalanced if $|\pi_\kappa| \leq \frac{\gamma n}{k}$. We may think of $\gamma$ as being the ratio of the smallest cluster size to the average cluster size. Fourth, because of the quantization, the iterations are no longer guaranteed to decrease the loss, so we have an early termination if the loss increases at any iteration. Note that the algorithm terminates almost surely.

\begin{wrapfigure}{r}{0.5\textwidth}
\vspace{-20pt}
 \begin{minipage}{.99\linewidth}
\begin{algorithm}[H]
   \caption{Quantized $k$-means (abridged)}
   \label{qkmeans_abridged}
\begin{algorithmic}
\footnotesize
   \STATE {\bfseries Input:} data matrix $D \in \mathbf{R}^{n \times d}$
   \STATE {\bfseries Parameters:} $k \in \mathbf{N}$, $T \in \mathbf{N}$, $\gamma \in (0,1)$, $\epsilon > 0$
   \STATE $c \gets k^{++}(D)$ // \textit{initialize centroids with $k$-means++}
   \STATE Save initial centroids:  $\textsf{save}(c)$.
   \STATE $L \gets k$-means loss of initial partition $\pi(c)$
   \FOR{$\tau = 1$ {\bfseries to} $T$}
        \STATE Store current centroids: $ c'\gets c$
        \STATE Compute centroids: $ c \gets c(\pi)$
        \STATE Apply correction to $\gamma$-imbalanced partitions
        \STATE Quantize to random $\epsilon$-lattice: $\hat{c} \gets Q(c;\theta)$
       \STATE Update partition: $\pi' \gets \pi(\hat{c})$
      \STATE Save state to metadata: $\textsf{save}(c,\theta,\hat{c},|\pi'|)$
       \STATE Compute loss $L'$
       \STATE \textbf{if} $L' < L$ \textbf{then} $(c, \pi, L) \gets (\hat{c}, \pi' , L')$ \textbf{else} \textbf{break} 
   \ENDFOR
   \\
   \textbf{return} $c$ //output final centroids as model
\end{algorithmic}
\end{algorithm}
\end{minipage}
\vspace{-10pt}
\end{wrapfigure}

Deletion in Q-$k$-means is straightforward. Using the metadata saved from training time, we can verify if deleting a specific datapoint would have resulted in a different \emph{quantized centroid} than was actually computed during training. If this is the case (or if the point to be deleted is one of randomly chosen initial centroids according to $k$-means++) we must retrain from scratch to satisfy the deletion request. Otherwise, we may satisfy deletion by updating our metadata to reflect the deletion of the specified datapoint, but we do not have to recompute the centroids. Q-$k$-means directly relies the principle of quantization to enable fast deletion in expectation. It is also worth noting that Q-$k$-means also leverages on the principle of linearity to recycle computation. Since centroid computation is linear in the datapoints, it is easy to determine the centroid update due to a removal at deletion time.

\paragraph{Deletion Time Complexity}
We turn our attention to an asymptotic time complexity analysis of Q-$k$-means deletion operation. Q-$k$-means supports deletion by quantizing the centroids, so they are stable to against small perturbations (caused by deletion of a point). 

\begin{theorem}

Let $D$ be a dataset on $[0,1]^d$ of size $n$. Fix parameters $T$, $k$, $\epsilon$, and $\gamma$ for Q-$k$-means. Then, Q-$k$-means supports $m$ deletions in time $O(m^2d^{5/2}/\epsilon)$ in expectation, with probability over the randomness in the quantization phase and $k$-means++ initialization.
\end{theorem}

The proof for the theorem is given in Appendix C.  The intuition is as follows. Centroids are computed by taking an average. With enough terms in an average, the effect of a small number of those terms is negligible. The removal of those terms from the average can be interpreted as a small perturbation to the centroid. If that small perturbation is on a scale far below the granularity of the quantizing $\epsilon$-lattice, then it is unlikely to change the quantized value of the centroid. Thus, beyond stability verification, no additional computation is required for a majority of deletion requests.
This result is in expectation with respect to the randomized initializations and randomized quantization phase, but is actually worst-case over all possible (normalized) dataset instances. The number of clusters $k$, iterations $T$, and cluster imbalance ratio $\gamma$ are usually small constants in many applications, and are treated as such here. Interestingly, for constant $m$ and $\epsilon$, the expected deletion time is independent of $n$ due to the stability probability increasing at the same rate as the problem size (see Appendix C). Deletion time for this method may not scale well in the high-dimensional setting. In the low-dimensional case, the most interesting interplay is between $\epsilon$, $n$, and $m$. To obtain as high-quality statistical performance as possible, it would be ideal if $\epsilon \rightarrow 0$ as $n \rightarrow \infty$. In this spirit, we can parameterize $\epsilon = n^{-\beta}$ for $\beta \in (0,1)$. We will use this parameterization for theoretical analysis of the online setting in Section 4.3.

\paragraph{Theoretical Statistical Performance}
We proceed to state a theoretical guarantee on statistical performance of Q-$k$-means, which complements the asymptotic time complexity bound of the deletion operation. Recall that the loss for a $k$-means problem instance is given by the sum of squared Euclidean distance from each datapoint to its nearest centroid. Let $\mathcal{L}^*$ be the optimal loss for a particular problem instance. Achieving the optimal solution is, in general, NP-Hard \cite{aloise2009np}. Instead, we can approximate it with $k$-means++, which achieves $\mathbf{E}\mathcal{L}^{++} \leq (8 \log k + 16)\mathcal{L}^{*}$ \cite{arthur2007k}.

\begin{corollary}
Let $\mathcal{L}$ be a random variable denoting the loss of Q-$k$-means on a particular problem instance of size $n$. Then $\mathbf{E}\mathcal{L} \leq (8 \log k + 16)\mathcal{L}^{*} + \epsilon\sqrt{nd(8 \log k + 16)\mathcal{L}^{*}} + \frac{1}{4}nd\epsilon^2$.
\end{corollary}

This corollary follows from the theoretical guarantees already known to apply to Lloyd's algorithm when initialized with $k$-means++, given by \cite{arthur2007k}. The proof can be found in Appendix C. We can interpret the bound by looking at the ratio of expected loss upper bounds for $k$-means++ and Q-$k$-means. If we assume our problem instance is generated by iid samples from some arbitrary non-atomic distribution, then it follows that $\mathcal{L}^* = O(n)$. Taking the loss ratio of upper bounds yields $\mathbf{E}\mathcal{L}/\mathbf{E}\mathcal{L}^{++} \leq 1 + O(d\epsilon^2 + \sqrt{d}\epsilon)$. Ensuring that $\epsilon << 1/\sqrt{d}$ implies the upper bound is as good as that of $k$-means++.

\subsection{Divide-and-Conquer \(k\)-Means}

\begin{wrapfigure}{r}{0.5\textwidth}
\vspace{-25pt}
 \begin{minipage}{.99\linewidth}
\begin{algorithm}[H]

   \caption{DC-$k$-means}    
   \label{dckmeans}
\begin{algorithmic}
\footnotesize
   \STATE {\bfseries Input:} data matrix $D \in \mathbf{R}^{n \times d}$
   \STATE {\bfseries Parameters:} $k \in \mathbf{N}$, $T \in \mathbf{N}$, tree width $w \in \mathbf{N}$, tree height $h \in \mathbf{N}$
   \STATE Initialize a $w$-ary tree of height $h$ such that each \textsf{node} has a pointer to a \textsf{dataset} and \textsf{centroids}
   \FOR{$i= 1$ {\bfseries to} $n$}
        \STATE Select a leaf $\textsf{node}$ uniformly at random
        \STATE $\textsf{node}.\textsf{dataset}.\textsf{add}(D_i)$
        \ENDFOR
   \FOR{$l = h$ {\bfseries down to} $0$}
        \FOR{ {\bfseries each} \textsf{node} in level $l$}
            \STATE $c \gets \textsf{k-means++}(\textsf{node}.\textsf{dataset},k,T)$
            \STATE $\textsf{node}.\textsf{centroids} \gets c$
            \IF{$ l > 0$}
            \STATE $\textsf{node}.\textsf{parent}.\textsf{dataset}.\textsf{add}(c)$
            \ELSE
            \STATE \textsf{save} all \textsf{nodes} as metadata
            \STATE \textbf{return} $c$ //model output
            \ENDIF
         \ENDFOR
        \ENDFOR
\end{algorithmic}
\end{algorithm}
\end{minipage}
\vspace{-15pt}

\end{wrapfigure}

We turn our attention to another variant of Lloyd's algorithm that also supports efficient deletion, albeit through quite different means. We refer to this algorithm as Divide-and-Conquer $k$-means (DC-$k$-means). At a high-level, DC-$k$-means works by partitioning the dataset into small sub-problems, solving each sub-problem as an independent $k$-means instance, and recursively merging the results. We present pseudo-code for DC-$k$-means here, and we refer the reader to Appendix B for pseudo-code of the deletion operation.

DC-$k$-means operates on a perfect $w$-ary tree of height $h$ (this could be relaxed to any rooted tree). The original dataset is \emph{partitioned} into each leaf in the tree as a uniform multinomial random variable with datapoints as trials and leaves as outcomes. At each of these leaves, we solve for some number of centroids via $k$-means++. When we merge leaves into their parent node, we construct a new dataset consisting of all the centroids from each leaf. Then, we compute new centroids at the parent via another instance of $k$-means++. For simplicity, we keep $k$ fixed throughout all of the sub-problems in the tree, but this could be relaxed. We make use of the tree hierarchy to \emph{modularize} the computation's dependence on the data. At deletion time, we need only to recompute the sub-problems from \emph{one} leaf up to the root. This observation allows us to support fast deletion operations.

Our method has close similarities to pre-existing distributed $k$-means algorithms \cite{qin2016distributed,peleg2000distributed,balcan2013distributed,bachem2017distributed,guha1998cure,bahmani2012scalable,zhao2009parallel}, but is in fact distinct (not only in that it is modified for deletion, but also in that it operates over general rooted trees). For simplicity, we restrict our discussion to only the simplest of divide-and-conquer trees. We focus on depth-1 trees with $w$ leaves where each leaf solves for $k$ centroids. This requires only one merge step with a root problem size of $kn/w$.

Analogous to how $\epsilon$ serves as a knob to trade-off between deletion efficiency and statistical performance in Q-$k$-means, for DC-$k$-means, we imagine that $w$ might also serve as a similar knob. For example, if $w = 1$, DC-$k$-means degenerates into canonical Lloyd's (as does Q-$k$-means as $\epsilon \rightarrow 0$). The dependence of statistical performance on tree width $w$ is less theoretically tractable than that of Q-$k$-means on $\epsilon$, but in Appendix D, we empirically show that statistical performance tends to decrease as $w$ increases, which is perhaps somewhat expected.

As we show in our experiments, depth-1 DC-$k$-means demonstrates an empirically compelling trade-off between deletion time and statistical performance. There are various other potential extensions of this algorithm, such as weighting centroids based on cluster mass as they propagate up the tree or exploring the statistical performance of deeper trees.

\paragraph{Deletion Time Complexity}
For ensuing asymptotic analysis, we may consider parameterizing tree width $w$ as $w = \Theta(n^\rho)$ for $\rho \in (0,1)$. As before, we treat $k$ and $T$ as small constants. Although intuitive, there are some technical minutia to account for to prove correctness and runtime for the DC-$k$-means deletion operation. The proof of Proposition 3.2 may be found in Appendix C. 

\begin{proposition}
Let $D$ be a dataset on $\mathbf{R}^d$ of size $n$. Fix parameters $T$ and $k$ for DC-$k$-means. Let $w = \Theta(n^\rho)$ and $\rho \in (0,1)$  Then, with a depth-1, $w$-ary divide-and-conquer tree, DC-$k$-means supports $m$ deletions in time $O(m\mathbf{max}\{n^{\rho},n^{1-\rho}\}d)$ in expectation with probability over the randomness in dataset partitioning.
\end{proposition}

\subsection{Amortized Runtime Complexity in Online Deletion Setting}

We state the amortized computation time for both of our algorithms in the online deletion setting defined in Section 3. We are in an asymptotic regime where the number of deletions $m = \Theta( n^\alpha)$ for $0 < \alpha < 1$ (see Appendix C for more details). Recall the $\Omega(\frac{n}{m})$ lower bound from Section 3. For a particular fractional power $\alpha$, an algorithm achieving the optimal asymptotic lower bound on amortized computation is said to be \emph{$\alpha$-deletion efficient}. This corresponds to achieving an amortized runtime of $O(n^{1-\alpha})$.  The following corollaries result from direct calculations which may be found in Appendix C. Note that Corollary 4.2.2 assumes DC-$k$-means is training sequentially. 

\begin{corollary}

With $\epsilon = \Theta(n^{-\beta}),$ for $0 < \beta < 1$, the Q-$k$-means algorithm is $\alpha$-deletion efficient in expectation if $\alpha \leq \frac{1-\beta}{2}$.

\end{corollary}

\begin{corollary}

With $w = \Theta(n^{\rho})$, for $0 < \rho < 1$, and a depth-1 $w$-ary divide-and-conquer tree, DC-$k$-means is $\alpha$-deletion efficient in expectation if $\alpha < 1- \textbf{max}\{1-\rho, \rho\}.$
\end{corollary}

\vspace{-5pt}
\section{Experiments}
\vspace{-5pt}
With a theoretical understanding in hand, we seek to empirically characterize the trade-off between runtime and performance for the proposed algorithms. In this section, we provide proof-of-concept for our algorithms by benchmarking their amortized runtimes and clustering quality on a simulated stream of online deletion requests. As a baseline, we use the canonical Lloyd's algorithm initialized by $k$-means++ seeding \cite{lloyd1982least,arthur2007k}. Following the broader literature, we refer to this baseline simply as $k$-means, and refer to our two proposed methods as Q-$k$-means and DC-$k$-means.

\paragraph{Datasets} We run our experiments on five real, publicly available datasets: \texttt{Celltype} ($N=12,009$, $D=10$, $K=4$)  \cite{han2018mapping}, \texttt{Covtype} ($N=15,120$, $D=52$, $K=7$) \cite{blackard1999comparative}, \texttt{MNIST} ($N=60,000$, $D=784$, $K=10$) \cite{lecun1998gradient},  \texttt{Postures} ($N=74,975$, $D=15$, $K=5$) \cite{gardner2014measuring,gardner20143d} , \texttt{Botnet} ($N=1,018,298$, $D=115$, $K=11$)\cite{meidan2018n}, and a synthetic dataset made from a Gaussian mixture model which we call \texttt{Gaussian} ($N=100,000$, $D=25$, $K=5$). We refer the reader to Appendix D for more details on the datasets. All datasets come with ground-truth labels as well. Although we do not make use of the labels at learning time, we can use them to evaluate the statistical quality of the clustering methods.

\paragraph{Online Deletion Benchmark} We simulate a stream of 1,000 deletion requests, selected uniformly at random and without replacement. An algorithm trains once, on the full dataset, and then runs its deletion operation to satisfy each request in the stream, producing an intermediate model at each request. For the canonical $k$-means baseline, deletions are satisfied by re-training from scratch. 

\paragraph{Protocol} To measure statistical performance, we evaluate with three metrics (see Section 5.1) that measure cluster quality. To measure deletion efficiency, we measure the wall-clock time to complete our online deletion benchmark. For both of our proposed algorithms, we always fix 10 iterations of Lloyd's, and all other parameters are selected with simple but effective heuristics (see Appendix D). This alleviates the need to tune them. To set a fair $k$-means baseline, when reporting runtime on the online deletion benchmark, we also fix 10 iterations of Lloyd's, but when reporting statistical performance metrics, we run until convergence. We run five replicates for each method on each dataset and include standard deviations with all our results. We refer the reader to Appendix D for more experimental details.

\subsection{Statistical Performance Metrics}
To evaluate clustering performance of our algorithms, the most obvious metric is the optimization loss of the $k$-means objective. Recall that this is the sum of square Euclidean distances from each datapoint to its nearest centroid. To thoroughly validate the statistical performance of our proposed algorithms, we additionally include two canonical clustering performance metrics.

\textbf{Silhouette Coefficient }\cite{rousseeuw1987silhouettes}: This coefficient measures a type of correlation (between -1 and +1) that captures how dense each cluster is and how well-separated different clusters are. The silhouette coefficient is computed without ground-truth labels, and uses only spatial information. Higher scores indicate denser, more well-separated clusters.

\textbf{Normalized Mutual Information (NMI)} \cite{vinh2010information,knops2006normalized}: This quantity measures the agreement of the assigned clusters to the ground-truth labels, up to permutation. NMI is upper bounded by 1, achieved by perfect assignments. Higher scores indicate better agreement between clusters and ground-truth labels.

\vspace{-5pt}
\subsection{Summary of Results}
\vspace{-5pt}

\begin{table}[h]
\begin{minipage}{0.3\textwidth}
\vspace{-5pt}
We summarize our key findings in four tables. In Tables 1-3, we report the statistical clustering performance of the 3 algorithms on each of the 6 datasets. In Table 1, we report the optimization loss ratios of our proposed methods over the $k$-means++ baseline. 
\end{minipage}\hfill 
\begin{minipage}{0.7\textwidth}
\centering
\small
\caption{Loss Ratio}
\begin{tabular}{l|c|c|c}
\toprule
 \bf Dataset  & \bf $k$-means \normalfont & \bf Q-$k$-means &\bf DC-$k$-means \normalfont \\
 \midrule
\bf Celltype & $1.0 \pm 0.0$ & $1.158 \pm 0.099$ & $1.439 \pm 0.157$ \\
\bf Covtype & $1.0 \pm 0.029$ & $1.033 \pm 0.017$ & $1.017 \pm 0.031$ \\
\bf MNIST & $1.0 \pm 0.002$ & $1.11 \pm 0.004$ & $1.014 \pm 0.003$ \\
\bf Postures & $1.0 \pm 0.004$ & $1.014 \pm 0.015$ & $1.034 \pm 0.017$ \\
\bf Gaussian & $1.0 \pm 0.014$ & $1.019 \pm 0.019$ & $1.003 \pm 0.014$ \\
\bf Botnet & $1.0 \pm 0.126$ & $1.018 \pm 0.014$ & $1.118 \pm 0.102$ \\
\bottomrule
\end{tabular}

\
\end{minipage}
\vspace{-5pt}
\end{table}

\begin{table}[h]
\begin{minipage}{0.3\textwidth}
\vspace{-5pt}

In Table 2, we report the silhouette coefficient for the clusters. In Table 3, we report the NMI. In Table 4, we report the amortized total runtime of training and deletion for each method. \textbf{Overall, we see that the statistical clustering performance of the three methods are competitive.} 

\end{minipage}\hfill 
\begin{minipage}{0.7\textwidth}
\centering
\small
\caption{Silhouette Coefficients (higher is better)}
\begin{tabular}{l|c|c|c}
\toprule
\bf Dataset  & \bf $k$-means \normalfont & \bf Q-$k$-means &\bf DC-$k$-means \normalfont \\
 \midrule
\bf Celltype & $0.384 \pm 0.001$ & $0.367 \pm 0.048$ & $0.422 \pm 0.057$ \\
\bf Covtype & $0.238 \pm 0.027$ & $0.203 \pm 0.026$ & $0.222 \pm 0.017$ \\
\bf Gaussian & $0.036 \pm 0.002$ & $0.031 \pm 0.002$ & $0.035 \pm 0.001$ \\
\bf Postures & $0.107 \pm 0.003$ & $0.107 \pm 0.004$ & $0.109 \pm 0.005$ \\
\bf Gaussian & $0.066 \pm 0.007$ & $0.053 \pm 0.003$ & $0.071 \pm 0.004$ \\
\bf Botnet & $0.583 \pm 0.042$ & $0.639 \pm 0.028$ & $0.627 \pm 0.046$ \\
\bottomrule
\end{tabular}
\label{table:silhouette}
\end{minipage}
\vspace{-5pt}
\end{table}

\begin{table}[h]
\begin{minipage}{0.3\textwidth}
\vspace{-5pt}
\textbf{Furthermore, we find that both proposed algorithms yield orders of magnitude of speedup.} As expected from the theoretical analysis, Q-$k$-means offers greater speed-ups when the dimension is lower relative to the sample size, whereas DC-$k$-means is more consistent across dimensionalities.

\end{minipage}\hfill 
\begin{minipage}{0.7\textwidth}

\centering
\small
\caption{Normalized Mutual Information (higher is better)}
\begin{tabular}{l|c|c|c}
\toprule
\bf Dataset  & \bf $k$-means \normalfont & \bf Q-$k$-means &\bf DC-$k$-means \normalfont \\
 
 \midrule
\bf Celltype& $0.36 \pm 0.0$ & $0.336 \pm 0.032$ & $0.294 \pm 0.067$ \\
\bf Covtype & $0.311 \pm 0.009$ & $0.332 \pm 0.024$ & $0.335 \pm 0.02$ \\
\bf MNIST & $0.494 \pm 0.006$ & $0.459 \pm 0.011$ & $0.494 \pm 0.004$ \\
\bf Gaussian & $0.319 \pm 0.024$ & $0.245 \pm 0.024$ & $0.318 \pm 0.024$ \\
\bf Postures & $0.163 \pm 0.018$ & $0.169 \pm 0.012$ & $0.173 \pm 0.011$ \\
\bf Botnet & $0.708 \pm 0.048$ & $0.73 \pm 0.015$ & $0.705 \pm 0.039$ \\
\bottomrule
\end{tabular}
\label{table:nmi}
\end{minipage}
\end{table}

\begin{table*}[h!]
\vspace{-15pt}
\centering
\small
\caption{Amortized Runtime in Online Deletion Benchmark (Train once + 1,000 Deletions) }
\small
\begin{tabular}{l|c|cc|cc}
\toprule
  
 & \bf $k$-means \normalfont  & \multicolumn{2}{c|}{\bf Q-$k$-means \normalfont} &
    \multicolumn{2}{c}{\bf  DC-$k$-means \normalfont} \\
\bf Dataset  & \bf Runtime (s) \normalfont & \bf Runtime (s) \normalfont  & \bf Speedup \normalfont & \bf Runtime (s) & \bf Speedup\\
\midrule
\bf Celltype & $4.241 \pm 0.248$ & $0.026 \pm 0.011$ & $163.286\times$ & $0.272 \pm 0.007$ & $15.6\times$ \\
\bf Covtype & $6.114 \pm 0.216$ & $0.454 \pm 0.276$ & $13.464\times$ & $0.469 \pm 0.021$ & $13.048\times$\\
\bf MNIST & $65.038 \pm 1.528$ & $29.386 \pm 0.728$ & $2.213\times$ & $2.562 \pm 0.056$ & $25.381\times$ \\
\bf Postures & $26.616 \pm 1.222$ & $0.413 \pm 0.305$ & $64.441\times$ & $1.17 \pm 0.398$ & $22.757\times$ \\
\bf Gaussian & $206.631 \pm 67.285$ & $0.393 \pm 0.104$ & $525.63\times$ & $5.992 \pm 0.269$ & $34.483\times$ \\
\bf Botnet & $607.784 \pm 64.687$ & $1.04 \pm 0.368$ & $584.416\times$ & $8.568 \pm 0.652$ & $70.939\times$ \\
\bottomrule
\end{tabular}
\label{table:speedup}
\vspace{-8pt}
\end{table*}

\begin{figure}[h!]
\vspace{-12pt}
\centering
\small
\includegraphics[width=0.95\textwidth]{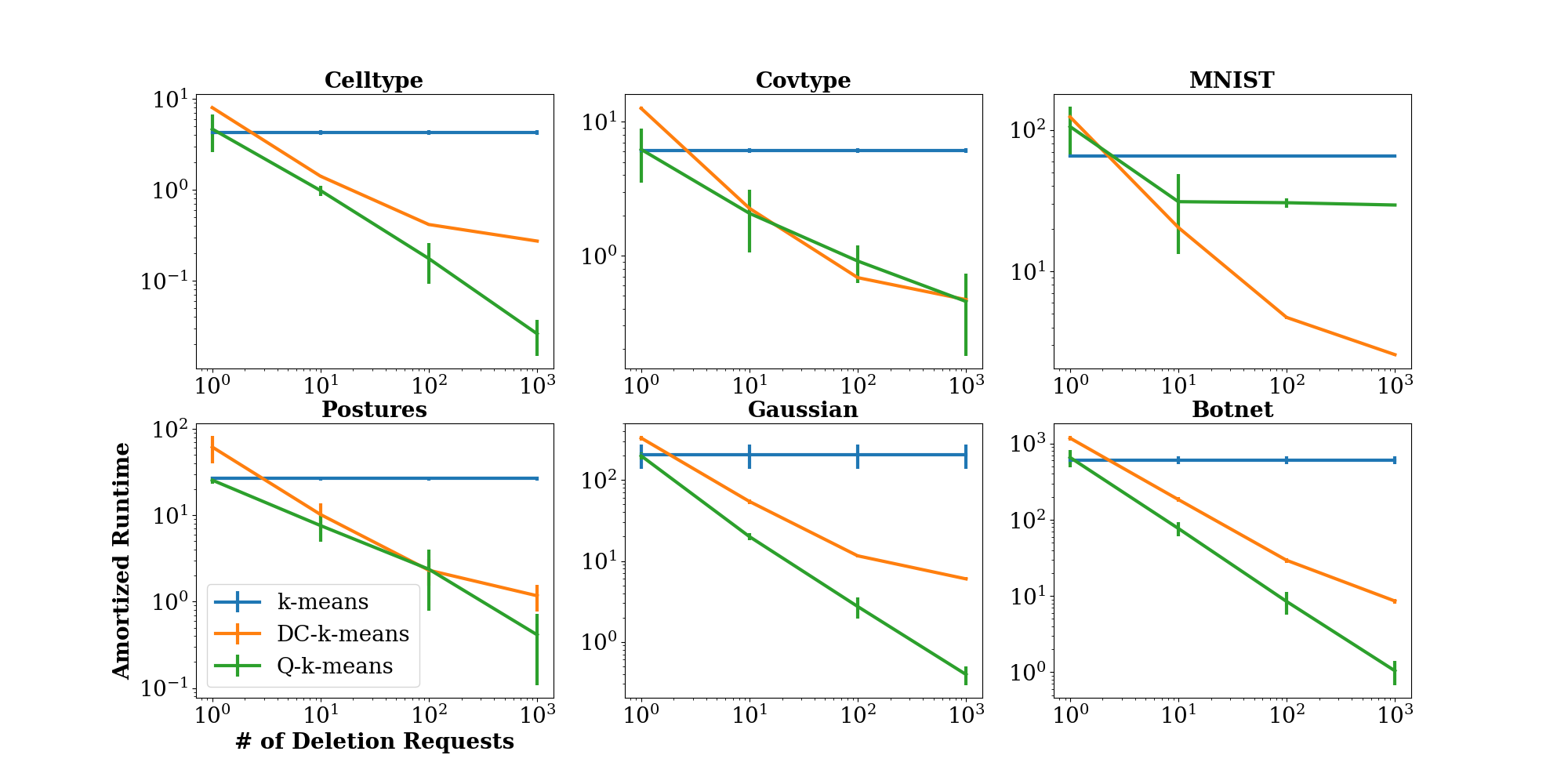}
\vspace{-15pt}
\caption{\small {Online deletion efficiency: \# of deletions vs. amortized runtime (secs) for 3 algorithms on 6 datasets.}}
\vspace{-10pt}
\end{figure}

In particular, note that \texttt{MNIST} has the highest $d/n$ ratio of the datasets we tried, followed by \texttt{Covtype}, These two datasets are, respectively, the datasets for which Q-$k$-means offers the least speedup. On the other hand, DC-$k$-means offers consistently increasing speedup as $n$ increases, for fixed $d$. Furthermore, we see that Q-$k$-means tends to have higher variance around its deletion efficiency, due to the randomness in centroid stabilization having a larger impact than the randomness in the dataset partitioning. We remark that 1,000 deletions is less than 10\% of every dataset we test on, and statistical performance remains virtually unchanged throughout the benchmark. In Figure 1, we plot the amortized runtime on the online deletion benchmark as a function of number of deletions in the stream. We refer the reader to Appendix D for supplementary experiments providing more detail on our methods.

\vspace{-5pt}
\section{Discussion}
\vspace{-5pt}
\label{sec:discussion}

At present, the main options for deletion efficient supervised methods are linear models, support vector machines, and non-parametric regressions. While our analysis here focuses on the concrete problem of clustering, we have proposed four design principles which we envision as the pillars of deletion efficient learning algorithms. We discuss the potential application of these methods to other supervised learning techniques.

\vspace{-2pt}
\paragraph{Segmented Regression} Segmented (or piece-wise) linear regression is a common relaxation of canonical regression models \cite{muggeo2016testing,muggeo2008segmented,muggeo2003estimating}. It should be possible to support a variant of segmented regression by combining Q-$k$-means with linear least squares regression. Each cluster could be given a separate linear model, trained only on the datapoints in said cluster. At deletion time, Q-$k$-means would likely keep the clusters stable, enabling a simple linear update to the model corresponding to the cluster from which the deleted point belonged.

\vspace{-5pt}
\paragraph{Kernel Regression} Kernel regressions in the style of random Fourier features  \cite{rahimi2008random} could be readily amended to support efficient deletions for large-scale supervised learning. Random features do not depend on data, and thus only the linear layer over the feature space requires updating for deletion. Furthermore, random Fourier feature methods have been shown to have affinity for quantization \cite{zhang2018low}.

\vspace{-5pt}
\paragraph{Decision Trees and Random Forests} Quantization is also a promising approach for decision trees. By quantizing or randomizing decision tree splitting criteria (such as in \cite{geurts2006extremely}) it seems possible to support efficient deletion. Furthermore, random forests have a natural affinity with bagging, which naturally can be used to impose modularity.   

\vspace{-5pt}
\paragraph{Deep Neural Networks and Stochastic Gradient Descent}
A line of research has observed the robustness of neural network training robustness to quantization and pruning \cite{vanhoucke2011improving,jouppi2017datacenter,gysel2018ristretto,rastegari2016xnor,courbariaux2014training,lin2016fixed}. It could be possible to leverage these techniques to quantize gradient updates during SGD-style optimization, enabling a notion of parameter stability analgous to that in Q-$k$-means. This would require larger batch sizes and fewer gradient steps in order to scale well. It is also possible that approximate deletion methods may be able to overcome shortcomings of exact deletion methods for large neural models.

\vspace{-5pt}
\section{Conclusion}
\vspace{-5pt}
In this work, we developed a notion of deletion efficiency for large-scale learning systems, proposed provably deletion efficient unsupervised clustering algorithms, and identified potential algorithmic principles that may enable deletion efficiency for other learning algorithms and paradigms.  We have only scratched the surface of understanding deletion efficiency in learning systems. Throughout, we made a number of simplifying assumptions, such that there is only one model and only one database in our system. We  also assumed that user-based deletion requests correspond to only a single data point. Understanding deletion efficiency in a system with many models and many databases, as well as complex user-to-data relationships, is an important direction for future work.  

\noindent \textbf{Acknowledgments:}  This research was partially supported by NSF Awards AF:1813049, CCF:1704417, and CCF 1763191, NIH R21 MD012867-01, NIH P30AG059307, an Office of Naval Research Young Investigator Award (N00014-18-1-2295), a seed grant from Stanford's Institute for Human-Centered AI, and the Chan-Zuckerberg Initiative. We would also like to thank I. Lemhadri, B. He, V. Bagaria, J. Thomas and anonymous reviewers for helpful discussion and feedback.  

\bibliographystyle{abbrv}
\bibliography{cites.bib}

\begin{thebibliography}{10}

\bibitem{abadi2016deep}
M.~Abadi, A.~Chu, I.~Goodfellow, H.~B. McMahan, I.~Mironov, K.~Talwar, and
  L.~Zhang.
\newblock Deep learning with differential privacy.
\newblock In {\em Proceedings of the 2016 ACM SIGSAC Conference on Computer and
  Communications Security}, pages 308--318. ACM, 2016.

\bibitem{abu2012learning}
Y.~S. Abu-Mostafa, M.~Magdon-Ismail, and H.-T. Lin.
\newblock {\em Learning from data}, volume~4.
\newblock AMLBook New York, NY, USA:, 2012.

\bibitem{aloise2009np}
D.~Aloise, A.~Deshpande, P.~Hansen, and P.~Popat.
\newblock Np-hardness of euclidean sum-of-squares clustering.
\newblock {\em Machine learning}, 75(2):245--248, 2009.

\bibitem{altman1992introduction}
N.~S. Altman.
\newblock An introduction to kernel and nearest-neighbor nonparametric
  regression.
\newblock {\em The American Statistician}, 46(3):175--185, 1992.

\bibitem{arthur2007k}
D.~Arthur and S.~Vassilvitskii.
\newblock k-means++: The advantages of careful seeding.
\newblock In {\em Proceedings of the eighteenth annual ACM-SIAM symposium on
  Discrete algorithms}, pages 1027--1035. Society for Industrial and Applied
  Mathematics, 2007.

\bibitem{atkeson1997locally}
C.~G. Atkeson, A.~W. Moore, and S.~Schaal.
\newblock Locally weighted learning for control.
\newblock In {\em Lazy learning}, pages 75--113. Springer, 1997.

\bibitem{bachem2017distributed}
O.~Bachem, M.~Lucic, and A.~Krause.
\newblock Distributed and provably good seedings for k-means in constant
  rounds.
\newblock In {\em Proceedings of the 34th International Conference on Machine
  Learning-Volume 70}, pages 292--300. JMLR. org, 2017.

\bibitem{bahmani2012scalable}
B.~Bahmani, B.~Moseley, A.~Vattani, R.~Kumar, and S.~Vassilvitskii.
\newblock Scalable k-means++.
\newblock {\em Proceedings of the VLDB Endowment}, 5(7):622--633, 2012.

\bibitem{balcan2013distributed}
M.-F.~F. Balcan, S.~Ehrlich, and Y.~Liang.
\newblock Distributed $ k $-means and $ k $-median clustering on general
  topologies.
\newblock In {\em Advances in Neural Information Processing Systems}, pages
  1995--2003, 2013.

\bibitem{berman1993optimization}
O.~Berman and N.~Ashrafi.
\newblock Optimization models for reliability of modular software systems.
\newblock {\em IEEE Transactions on Software Engineering}, 19(11):1119--1123,
  1993.

\bibitem{Birattari:1999:LLM:340534.340673}
M.~Birattari, G.~Bontempi, and H.~Bersini.
\newblock Lazy learning meets the recursive least squares algorithm.
\newblock In {\em Proceedings of the 1998 Conference on Advances in Neural
  Information Processing Systems II}, pages 375--381, Cambridge, MA, USA, 1999.
  MIT Press.

\bibitem{blackard1999comparative}
J.~A. Blackard and D.~J. Dean.
\newblock Comparative accuracies of artificial neural networks and discriminant
  analysis in predicting forest cover types from cartographic variables.
\newblock {\em Computers and electronics in agriculture}, 24(3):131--151, 1999.

\bibitem{bogdanov2018implementation}
D.~Bogdanov, L.~Kamm, S.~Laur, and V.~Sokk.
\newblock Implementation and evaluation of an algorithm for cryptographically
  private principal component analysis on genomic data.
\newblock {\em IEEE/ACM transactions on computational biology and
  bioinformatics}, 15(5):1427--1432, 2018.

\bibitem{bonawitz2017practical}
K.~Bonawitz, V.~Ivanov, B.~Kreuter, A.~Marcedone, H.~B. McMahan, S.~Patel,
  D.~Ramage, A.~Segal, and K.~Seth.
\newblock Practical secure aggregation for privacy-preserving machine learning.
\newblock In {\em Proceedings of the 2017 ACM SIGSAC Conference on Computer and
  Communications Security}, pages 1175--1191. ACM, 2017.

\bibitem{bontempi2001local}
G.~Bontempi, H.~Bersini, and M.~Birattari.
\newblock The local paradigm for modeling and control: from neuro-fuzzy to lazy
  learning.
\newblock {\em Fuzzy sets and systems}, 121(1):59--72, 2001.

\bibitem{bost2015machine}
R.~Bost, R.~A. Popa, S.~Tu, and S.~Goldwasser.
\newblock Machine learning classification over encrypted data.
\newblock In {\em NDSS}, 2015.

\bibitem{bottou1998online}
L.~Bottou.
\newblock Online learning and stochastic approximations.
\newblock {\em On-line learning in neural networks}, 17(9):142, 1998.

\bibitem{cao2015towards}
Y.~Cao and J.~Yang.
\newblock Towards making systems forget with machine unlearning.
\newblock In {\em 2015 IEEE Symposium on Security and Privacy}, pages 463--480.
  IEEE, 2015.

\bibitem{cauwenberghs2001incremental}
G.~Cauwenberghs and T.~Poggio.
\newblock Incremental and decremental support vector machine learning.
\newblock In {\em Advances in neural information processing systems}, pages
  409--415, 2001.

\bibitem{chaudhuri2011differentially}
K.~Chaudhuri, C.~Monteleoni, and A.~D. Sarwate.
\newblock Differentially private empirical risk minimization.
\newblock {\em Journal of Machine Learning Research}, 12(Mar):1069--1109, 2011.

\bibitem{chaudhuri2013near}
K.~Chaudhuri, A.~D. Sarwate, and K.~Sinha.
\newblock A near-optimal algorithm for differentially-private principal
  components.
\newblock {\em The Journal of Machine Learning Research}, 14(1):2905--2943,
  2013.

\bibitem{coomans1982alternative}
D.~Coomans and D.~L. Massart.
\newblock Alternative k-nearest neighbour rules in supervised pattern
  recognition: Part 1. k-nearest neighbour classification by using alternative
  voting rules.
\newblock {\em Analytica Chimica Acta}, 136:15--27, 1982.

\bibitem{righttobeforgotten}
{Council of European Union}.
\newblock Council regulation (eu) no 2012/0011, 2014.
\newblock
  \url{https://eur-lex.europa.eu/legal-content/EN/TXT/?uri=CELEX:52012PC0011}.

\bibitem{gdpr}
{Council of European Union}.
\newblock Council regulation (eu) no 2016/678, 2014.
\newblock \url{https://eur-lex.europa.eu/eli/reg/2016/679/oj}.

\bibitem{courbariaux2014training}
M.~Courbariaux, Y.~Bengio, and J.-P. David.
\newblock Training deep neural networks with low precision multiplications.
\newblock {\em arXiv preprint arXiv:1412.7024}, 2014.

\bibitem{cover2012elements}
T.~M. Cover and J.~A. Thomas.
\newblock {\em Elements of information theory}.
\newblock John Wiley \& Sons, 2012.

\bibitem{regression}
R.~E.~W. D.~A.~Belsley, E.~Kuh.
\newblock {\em Regression Diagnostics: Identifying Influential Data and Sources
  of Collinearity}.
\newblock John Wiley \& Sons, Inc., New York, NY, USA, 1980.

\bibitem{dasgupta2003elementary}
S.~Dasgupta and A.~Gupta.
\newblock An elementary proof of a theorem of johnson and lindenstrauss.
\newblock {\em Random Structures and Algorithms}, 22(1):60--65, 2003.

\bibitem{devroye1979distribution}
L.~Devroye and T.~Wagner.
\newblock Distribution-free performance bounds for potential function rules.
\newblock {\em IEEE Transactions on Information Theory}, 25(5):601--604, 1979.

\bibitem{dwork2014algorithmic}
C.~Dwork, A.~Roth, et~al.
\newblock The algorithmic foundations of differential privacy.
\newblock {\em Foundations and Trends in Theoretical Computer Science},
  9(3--4):211--407, 2014.

\bibitem{erkin2012generating}
Z.~Erkin, T.~Veugen, T.~Toft, and R.~L. Lagendijk.
\newblock Generating private recommendations efficiently using homomorphic
  encryption and data packing.
\newblock {\em IEEE transactions on information forensics and security},
  7(3):1053--1066, 2012.

\bibitem{friedman2001elements}
J.~Friedman, T.~Hastie, and R.~Tibshirani.
\newblock {\em The elements of statistical learning}.
\newblock Number~10. Springer series in statistics New York, 2001.

\bibitem{galinsky2016population}
K.~J. Galinsky, P.-R. Loh, S.~Mallick, N.~J. Patterson, and A.~L. Price.
\newblock Population structure of uk biobank and ancient eurasians reveals
  adaptation at genes influencing blood pressure.
\newblock {\em The American Journal of Human Genetics}, 99(5):1130--1139, 2016.

\bibitem{gardner20143d}
A.~Gardner, C.~A. Duncan, J.~Kanno, and R.~Selmic.
\newblock 3d hand posture recognition from small unlabeled point sets.
\newblock In {\em 2014 IEEE International Conference on Systems, Man, and
  Cybernetics (SMC)}, pages 164--169. IEEE, 2014.

\bibitem{gardner2014measuring}
A.~Gardner, J.~Kanno, C.~A. Duncan, and R.~Selmic.
\newblock Measuring distance between unordered sets of different sizes.
\newblock In {\em Proceedings of the IEEE Conference on Computer Vision and
  Pattern Recognition}, pages 137--143, 2014.

\bibitem{geurts2006extremely}
P.~Geurts, D.~Ernst, and L.~Wehenkel.
\newblock Extremely randomized trees.
\newblock {\em Machine learning}, 63(1):3--42, 2006.

\bibitem{ghorbani2019data}
A.~Ghorbani and J.~Zou.
\newblock Data shapley: Equitable valuation of data for machine learning.
\newblock {\em arXiv preprint arXiv:1904.02868}, 2019.

\bibitem{gray1998quantization}
R.~M. Gray and D.~L. Neuhoff.
\newblock Quantization.
\newblock {\em IEEE transactions on information theory}, 44(6):2325--2383,
  1998.

\bibitem{guha1998cure}
S.~Guha, R.~Rastogi, and K.~Shim.
\newblock Cure: an efficient clustering algorithm for large databases.
\newblock In {\em ACM Sigmod Record}, pages 73--84. ACM, 1998.

\bibitem{gysel2018ristretto}
P.~Gysel, J.~Pimentel, M.~Motamedi, and S.~Ghiasi.
\newblock Ristretto: A framework for empirical study of resource-efficient
  inference in convolutional neural networks.
\newblock {\em IEEE Transactions on Neural Networks and Learning Systems},
  2018.

\bibitem{hammarling2008updating}
S.~Hammarling and C.~Lucas.
\newblock Updating the qr factorization and the least squares problem.
\newblock {\em Tech. Report, The University of Manchester (2008)}, 2008.

\bibitem{han2018mapping}
X.~Han, R.~Wang, Y.~Zhou, L.~Fei, H.~Sun, S.~Lai, A.~Saadatpour, Z.~Zhou,
  H.~Chen, F.~Ye, et~al.
\newblock Mapping the mouse cell atlas by microwell-seq.
\newblock {\em Cell}, 172(5):1091--1107, 2018.

\bibitem{higham2002accuracy}
N.~J. Higham.
\newblock {\em Accuracy and stability of numerical algorithms}, volume~80.
\newblock Siam, 2002.

\bibitem{hinneburg2007denclue}
A.~Hinneburg and H.-H. Gabriel.
\newblock Denclue 2.0: Fast clustering based on kernel density estimation.
\newblock In {\em International symposium on intelligent data analysis}, pages
  70--80. Springer, 2007.

\bibitem{johnson1984extensions}
W.~B. Johnson and J.~Lindenstrauss.
\newblock Extensions of lipschitz mappings into a hilbert space.
\newblock {\em Contemporary mathematics}, 26(189-206):1, 1984.

\bibitem{jouppi2017datacenter}
N.~P. Jouppi, C.~Young, N.~Patil, D.~Patterson, G.~Agrawal, R.~Bajwa, S.~Bates,
  S.~Bhatia, N.~Boden, A.~Borchers, et~al.
\newblock In-datacenter performance analysis of a tensor processing unit.
\newblock In {\em Computer Architecture (ISCA), 2017 ACM/IEEE 44th Annual
  International Symposium on}, pages 1--12. IEEE, 2017.

\bibitem{kearns1999algorithmic}
M.~Kearns and D.~Ron.
\newblock Algorithmic stability and sanity-check bounds for leave-one-out
  cross-validation.
\newblock {\em Neural computation}, 11(6):1427--1453, 1999.

\bibitem{knoblauch2008closed}
A.~Knoblauch.
\newblock Closed-form expressions for the moments of the binomial probability
  distribution.
\newblock {\em SIAM Journal on Applied Mathematics}, 69(1):197--204, 2008.

\bibitem{knops2006normalized}
Z.~F. Knops, J.~A. Maintz, M.~A. Viergever, and J.~P. Pluim.
\newblock Normalized mutual information based registration using k-means
  clustering and shading correction.
\newblock {\em Medical image analysis}, 10(3):432--439, 2006.

\bibitem{kutin2002almost}
S.~Kutin and P.~Niyogi.
\newblock Almost-everywhere algorithmic stability and generalization error:
  Tech. rep.
\newblock Technical report, TR-2002-03: University of Chicago, Computer Science
  Department, 2002.

\bibitem{lecun1998gradient}
Y.~LeCun, L.~Bottou, Y.~Bengio, P.~Haffner, et~al.
\newblock Gradient-based learning applied to document recognition.
\newblock {\em Proceedings of the IEEE}, 86(11):2278--2324, 1998.

\bibitem{lin2016fixed}
D.~Lin, S.~Talathi, and S.~Annapureddy.
\newblock Fixed point quantization of deep convolutional networks.
\newblock In {\em International Conference on Machine Learning}, pages
  2849--2858, 2016.

\bibitem{lloyd1982least}
S.~Lloyd.
\newblock Least squares quantization in pcm.
\newblock {\em IEEE transactions on information theory}, 28(2):129--137, 1982.

\bibitem{lowe1999object}
D.~G. Lowe et~al.
\newblock Object recognition from local scale-invariant features.
\newblock In {\em ICCV}, number~2, pages 1150--1157, 1999.

\bibitem{statisticalcomputation}
J.~H. Maindonald.
\newblock {\em Statistical Computation}.
\newblock John Wiley \& Sons, Inc., New York, NY, USA, 1984.

\bibitem{meidan2018n}
Y.~Meidan, M.~Bohadana, Y.~Mathov, Y.~Mirsky, A.~Shabtai, D.~Breitenbacher, and
  Y.~Elovici.
\newblock N-baiot—network-based detection of iot botnet attacks using deep
  autoencoders.
\newblock {\em IEEE Pervasive Computing}, 17(3):12--22, 2018.

\bibitem{muggeo2003estimating}
V.~M. Muggeo.
\newblock Estimating regression models with unknown break-points.
\newblock {\em Statistics in medicine}, 22(19):3055--3071, 2003.

\bibitem{muggeo2016testing}
V.~M. Muggeo.
\newblock Testing with a nuisance parameter present only under the alternative:
  a score-based approach with application to segmented modelling.
\newblock {\em Journal of Statistical Computation and Simulation},
  86(15):3059--3067, 2016.

\bibitem{muggeo2008segmented}
V.~M. Muggeo et~al.
\newblock Segmented: an r package to fit regression models with broken-line
  relationships.
\newblock {\em R news}, 8(1):20--25, 2008.

\bibitem{mukherjee2006learning}
S.~Mukherjee, P.~Niyogi, T.~Poggio, and R.~Rifkin.
\newblock Learning theory: stability is sufficient for generalization and
  necessary and sufficient for consistency of empirical risk minimization.
\newblock {\em Advances in Computational Mathematics}, 25(1-3):161--193, 2006.

\bibitem{nadaraya1964estimating}
E.~A. Nadaraya.
\newblock On estimating regression.
\newblock {\em Theory of Probability \& Its Applications}, 9(1):141--142, 1964.

\bibitem{nikolaenko2013privacy}
V.~Nikolaenko, U.~Weinsberg, S.~Ioannidis, M.~Joye, D.~Boneh, and N.~Taft.
\newblock Privacy-preserving ridge regression on hundreds of millions of
  records.
\newblock In {\em Security and Privacy (SP), 2013 IEEE Symposium on}, pages
  334--348. IEEE, 2013.

\bibitem{ohrimenko2016oblivious}
O.~Ohrimenko, F.~Schuster, C.~Fournet, A.~Mehta, S.~Nowozin, K.~Vaswani, and
  M.~Costa.
\newblock Oblivious multi-party machine learning on trusted processors.
\newblock In {\em USENIX Security Symposium}, pages 619--636, 2016.

\bibitem{papernot2016semi}
N.~Papernot, M.~Abadi, U.~Erlingsson, I.~Goodfellow, and K.~Talwar.
\newblock Semi-supervised knowledge transfer for deep learning from private
  training data.
\newblock {\em arXiv preprint arXiv:1610.05755}, 2016.

\bibitem{pedregosa2011scikit}
F.~Pedregosa, G.~Varoquaux, A.~Gramfort, V.~Michel, B.~Thirion, O.~Grisel,
  M.~Blondel, P.~Prettenhofer, R.~Weiss, V.~Dubourg, et~al.
\newblock Scikit-learn: Machine learning in python.
\newblock {\em Journal of machine learning research}, 12(Oct):2825--2830, 2011.

\bibitem{scikit-learn}
F.~Pedregosa, G.~Varoquaux, A.~Gramfort, V.~Michel, B.~Thirion, O.~Grisel,
  M.~Blondel, P.~Prettenhofer, R.~Weiss, V.~Dubourg, J.~Vanderplas, A.~Passos,
  D.~Cournapeau, M.~Brucher, M.~Perrot, and E.~Duchesnay.
\newblock Scikit-learn: Machine learning in {P}ython.
\newblock {\em Journal of Machine Learning Research}, 12:2825--2830, 2011.

\bibitem{peleg2000distributed}
D.~Peleg.
\newblock Distributed computing.
\newblock {\em SIAM Monographs on discrete mathematics and applications},
  5:1--1, 2000.

\bibitem{poggio2004general}
T.~Poggio, R.~Rifkin, S.~Mukherjee, and P.~Niyogi.
\newblock General conditions for predictivity in learning theory.
\newblock {\em Nature}, 428(6981):419, 2004.

\bibitem{qin2016distributed}
J.~Qin, W.~Fu, H.~Gao, and W.~X. Zheng.
\newblock Distributed $ k $-means algorithm and fuzzy $ c $-means algorithm for
  sensor networks based on multiagent consensus theory.
\newblock {\em IEEE transactions on cybernetics}, 47(3):772--783, 2016.

\bibitem{rahimi2008random}
A.~Rahimi and B.~Recht.
\newblock Random features for large-scale kernel machines.
\newblock In {\em Advances in neural information processing systems}, pages
  1177--1184, 2008.

\bibitem{rastegari2016xnor}
M.~Rastegari, V.~Ordonez, J.~Redmon, and A.~Farhadi.
\newblock Xnor-net: Imagenet classification using binary convolutional neural
  networks.
\newblock In {\em European Conference on Computer Vision}, pages 525--542.
  Springer, 2016.

\bibitem{rousseeuw1987silhouettes}
P.~J. Rousseeuw.
\newblock Silhouettes: a graphical aid to the interpretation and validation of
  cluster analysis.
\newblock {\em Journal of computational and applied mathematics}, 20:53--65,
  1987.

\bibitem{schellekens2018quantized}
V.~Schellekens and L.~Jacques.
\newblock Quantized compressive k-means.
\newblock {\em IEEE Signal Processing Letters}, 25(8):1211--1215, 2018.

\bibitem{schelteramnesia}
S.~Schelter.
\newblock “amnesia”--towards machine learning models that can forget user
  data very fast.
\newblock In {\em 1st International Workshop on Applied AI for Database Systems
  and Applications (AIDB’19)}, 2019.

\bibitem{schomm2013marketplaces}
F.~Schomm, F.~Stahl, and G.~Vossen.
\newblock Marketplaces for data: an initial survey.
\newblock {\em ACM SIGMOD Record}, 42(1):15--26, 2013.

\bibitem{schrauwen2007overview}
B.~Schrauwen, D.~Verstraeten, and J.~Van~Campenhout.
\newblock An overview of reservoir computing: theory, applications and
  implementations.
\newblock In {\em Proceedings of the 15th european symposium on artificial
  neural networks. p. 471-482 2007}, pages 471--482, 2007.

\bibitem{shalev2010learnability}
S.~Shalev-Shwartz, O.~Shamir, N.~Srebro, and K.~Sridharan.
\newblock Learnability, stability and uniform convergence.
\newblock {\em Journal of Machine Learning Research}, 11(Oct):2635--2670, 2010.

\bibitem{shannon1949communication}
C.~E. Shannon.
\newblock Communication theory of secrecy systems.
\newblock {\em Bell system technical journal}, 28(4):656--715, 1949.

\bibitem{sudlow2015uk}
C.~Sudlow, J.~Gallacher, N.~Allen, V.~Beral, P.~Burton, J.~Danesh, P.~Downey,
  P.~Elliott, J.~Green, M.~Landray, et~al.
\newblock Uk biobank: an open access resource for identifying the causes of a
  wide range of complex diseases of middle and old age.
\newblock {\em PLoS medicine}, 12(3):e1001779, 2015.

\bibitem{truong2012data}
H.-L. Truong, M.~Comerio, F.~De~Paoli, G.~Gangadharan, and S.~Dustdar.
\newblock Data contracts for cloud-based data marketplaces.
\newblock {\em International Journal of Computational Science and Engineering},
  7(4):280--295, 2012.

\bibitem{tsai2014incremental}
C.-H. Tsai, C.-Y. Lin, and C.-J. Lin.
\newblock Incremental and decremental training for linear classification.
\newblock In {\em Proceedings of the 20th ACM SIGKDD international conference
  on Knowledge discovery and data mining}, pages 343--352. ACM, 2014.

\bibitem{van2011numpy}
S.~Van Der~Walt, S.~C. Colbert, and G.~Varoquaux.
\newblock The numpy array: a structure for efficient numerical computation.
\newblock {\em Computing in Science \& Engineering}, 13(2):22, 2011.

\bibitem{van1983matrix}
C.~F. Van~Loan and G.~H. Golub.
\newblock {\em Matrix computations}.
\newblock Johns Hopkins University Press, 1983.

\bibitem{vanhoucke2011improving}
V.~Vanhoucke, A.~Senior, and M.~Z. Mao.
\newblock Improving the speed of neural networks on cpus.
\newblock Citeseer.

\bibitem{veale2018algorithms}
M.~Veale, R.~Binns, and L.~Edwards.
\newblock Algorithms that remember: model inversion attacks and data protection
  law.
\newblock {\em Philosophical Transactions of the Royal Society A: Mathematical,
  Physical and Engineering Sciences}, 376(2133):20180083, 2018.

\bibitem{villaronga2018humans}
E.~F. Villaronga, P.~Kieseberg, and T.~Li.
\newblock Humans forget, machines remember: Artificial intelligence and the
  right to be forgotten.
\newblock {\em Computer Law \& Security Review}, 34(2):304--313, 2018.

\bibitem{vinh2010information}
N.~X. Vinh, J.~Epps, and J.~Bailey.
\newblock Information theoretic measures for clusterings comparison: Variants,
  properties, normalization and correction for chance.
\newblock {\em Journal of Machine Learning Research}, 11(Oct):2837--2854, 2010.

\bibitem{Webb2010}
G.~I. Webb.
\newblock {\em Lazy Learning}, pages 571--572.
\newblock Springer US, 2010.

\bibitem{yin2012self}
J.~Yin and Y.~Meng.
\newblock Self-organizing reservior computing with dynamically regulated
  cortical neural networks.
\newblock In {\em The 2012 International Joint Conference on Neural Networks
  (IJCNN)}, pages 1--7. IEEE, 2012.

\bibitem{zeb2017updating}
S.~Zeb and M.~Yousaf.
\newblock Updating qr factorization procedure for solution of linear least
  squares problem with equality constraints.
\newblock {\em Journal of inequalities and applications}, 2017(1):281, 2017.

\bibitem{zhang2018low}
J.~Zhang, A.~May, T.~Dao, and C.~R{\'e}.
\newblock Low-precision random fourier features for memory-constrained kernel
  approximation.
\newblock {\em arXiv preprint arXiv:1811.00155}, 2018.

\bibitem{zhao2009parallel}
W.~Zhao, H.~Ma, and Q.~He.
\newblock Parallel k-means clustering based on mapreduce.
\newblock In {\em IEEE International Conference on Cloud Computing}, pages
  674--679. Springer, 2009.

\end{thebibliography}

\newpage
\appendix
\section{Supplementary Materials}

Here we provide material supplementary to the main text. While some of the material provided here may be somewhat redundant, it also contains technical minutia perhaps too detailed for the main body.

\subsection{Online Data Deletion}

We precisely define the notion of a \emph{learning algorithm} for theoretical discussion in the context of data deletion. 

\begin{definition}{Learning Algorithm}

A \emph{learning} algorithm $A$ is an algorithm (on some standard model of computation) taking values in some hypothesis space and metadata space $\mathcal{H} \times \mathcal{M}$ based on an input dataset $D$. Learning algorithm $A$ may be randomized, implying a conditional distribution over $\mathcal{H} \times \mathcal{M}$ given $D$. Finally, learning algorithms must process each datapoint in $D$ at least once, and are constrained to sequential computation only, yielding a runtime bounded by $\Omega(n)$.
\end{definition}

We re-state the definition of data deletion. We distinguish between a \emph{deletion operation} and a \emph{robust deletion operation}. We focus on the former throughout our main body, as it is appropriate for average-case analysis in a non-security context. We use $=_d$ to denote distributional equality.

\begin{definition}{Data Deletion Operation}

Fix any dataset $D$ and learning algorithm $A$. Operation $R_A$ is a \emph{deletion operation} for $A$ if $R_A(D,A(D),i) =_d A(D_{-i})$ for any $i$ selected independently of $A(D)$.
\end{definition}

For notational simplicity, we may let $R_A$ refer to an entire sequence of deletions ($\Delta = \{i_1,i_2,...,i_m\}$) by writing  $R_A(D,A(D),\Delta)$. This notation means the output of a sequence of applications of $R_A$ to each $i$ in deletion sequence $\Delta$. We also may drop the dependence on $A$ when it is understood for which $A$ the deletion operation $R$ corresponds. We also drop the arguments for $A$ and $R$ when they are understood from context. For example, when dataset $D$ can be inferred from context, we let $A_{-i}$ directly mean $A(D_{-i})$ and when  and deletion stream $\Delta$ can be inferred, we let $R$ directly mean $R(D,A(D),\Delta)$.

Our definition is somewhat analogous to \emph{information-theoretic} (or perfect) secrecy in cryptography \cite{shannon1949communication}. Much like in cryptography, it is possible to relax to weaker notions -- for example, by statistically approximating deletion and bounding the amount of computation some hypothetical adversary could use to determine if a genuine deletion took place. Such relaxations are required for encryption algorithms because perfect secrecy can only be achieved via one-time pad \cite{shannon1949communication}. In the context of deletion operations, retraining from scratch is, at least slightly, analogous to one-time pad encryption: both are simple solutions that satisfy distributional equality requirements, but both solutions are impractical. However, unlike in encryption, when it comes to deletion, we can, in fact, at least for some learning algorithms, find deletion operations that would be both practical and perfect.

The upcoming \emph{robust} definition may be of more interest in a worst-case, security setting. In such a setting, an adaptive adversary makes deletion requests while also having perfect eavesdropping capabilities to the server (or at least the internal state of the learning algorithm, model and metadata).  

\begin{definition}{Robust Data Deletion Operation}

Fix any dataset $D$ and learning algorithm $A$. Operation $R_A$ is a \emph{robust deletion operation} if $R_A(D,A(D),i) =_d A(D_{-i})$ in distribution, for any $i$, perhaps selected by an adversarial agent with knowledge of $A(D)$.
\end{definition}

To illustrate the difference between these two definitions, consider Q-$k$-means and DC-$k$-means. Assume an adversary has compromised the server with read-access and gained knowledge of the algorithm's internal state. Further assume that said adversary may issue deletion requests. Such a powerful adversary could compromise the exactness of DC-$k$-means deletion operations by deleting datapoints from \emph{specific} leaves. For example, if the adversary always deletes datapoints partitioned to the first leaf, then the number of datapoints assigned to each leaf is no longer uniform or independent of deletion requests. In principle, this, at least rigorously speaking, violates equality in distribution. Note that this can only occur if requests are somehow dependent on the partition. However, despite an adversary being able to compromise the correctness of the deletion operation, it cannot compromise the efficiency.  That is because efficiency depends on the maximum number of datapoints partitioned to a particular leaf, a quantity which is decided randomly without input from the adversary. 

In the case of Q-$k$-means we can easily see the deletion is robust to the adversary by the enforced equality of outcome imposed by the deletion operation. However, an adversary with knowledge of algorithm state could make the Q-$k$-means deletion operation entirely inefficient by always deleting an initial centroid. This causes every single deletion to be satisfied by retraining from scratch. From the security perspective, it could be of interest to study deletion operations that are both robust and efficient.

We continue by defining the \emph{online} data deletion setting in the \emph{average-case}.

\begin{definition}{Online Data Deletion (Average-Case)}

We may formally define the runtime in the online deletion setting as the expected runtime of Algorithm 3. We amortize the total runtime by $m$.

\begin{algorithm}[h!]
   \caption{Online Data Deletion}
   \label{alg:online_del}
\begin{algorithmic}
\footnotesize
  \STATE {\bfseries Input:} Dataset $D$, learning algorithm $A$, deletion operation $R$   
  \STATE {\bfseries Parameters:} $\alpha \in (0,1)$
  \STATE $\mu \gets A(D)$
  \STATE $m \gets \Theta(|D|^\alpha)$ 
 \FOR{$\tau = 1$ {\bfseries to} $m$}
    \STATE $i \gets \textbf{Unif}[1,...,|D|]$ //constant time
    \STATE $\mu \gets R(D,\mu,i)$
    \STATE $D \gets D_{-i}$ //constant time
    \ENDFOR

\end{algorithmic}
\end{algorithm}

\end{definition}

$\textbf{Fractional power regime.}$ For dataset size $n$, when $m = \Theta(n^\alpha)$ for $0 < \alpha < 1$, we say we are in the \emph{fractional power regime}. For $k$-means, our proposed algorithms achieve the ideal lower bound for small enough $\alpha$, but not for all $\alpha$ in $(0,1)$.

Online data deletion is interesting in both the \emph{average-case} setting (treated here), where the indices $\{i_1,\ldots,i_m\}$ are chosen uniformly and independently without replacement, as well as in a \emph{worst-case} setting, where the sequence of indices is computed adversarially (left for future work). It may also be practical to include a bound on the amount of memory available to the data deletion operation and model (including metadata) as an additional constraint.

\begin{definition}{Deletion Efficient Learning Algorithm}

Recall the $\Omega(n/m)$ lower bound on amortized computation for any sequential learning algorithm in the online deletion setting (Section 2). Given some fractional power scaling $m = \Theta(n^\alpha)$, we say an algorithm $A$ is $\alpha$-deletion efficient if it runs Algorithm 3 in amortized time $O(n^{1-\alpha})$.

\end{definition}

\paragraph{Inference Time}
Of course, lazy learning and non-parametric techniques are a clear exception to our notions of learning algorithm. For these methods, data is processed at inference time rather than training time -- a more complete study of the systems trade-offs between training time, inference time, and deletion time is left for future work.

\subsection{Approximate Data Deletion}

We present one possible relaxation from exact to \emph{approximate} data deletion.

\begin{definition}{Approximate deletion}

We say that such a data deletion operation $R_A$ is an $\delta$-deletion for algorithm $A$ if, for all $D$ and for every measurable subset $S\subseteq\mathcal{H} \times M$:

$$\textbf{Pr}[A(D_{-i}) \in S | D_{-i}] \geq \delta  \textbf{Pr}[R_A(D,A(D),i) \in S | D_{-i}]$$  

\end{definition}

The above definition corresponds to requiring that the probability that the data deletion operation returns a model in some specified set, $S$, cannot be more than a $\delta^{-1}$ factor larger than the probability that algorithm $A$ retrained on the dataset $D_{-i}$ returns a model in that set.  We note that the above definition \emph{does} allow for the possibility that some outcomes that have positive probability under $A(D_{-i})$ have zero probability under the deletion operation.  In such a case, an observer could conclude that the model was returned by running $A$ from scratch.

\subsubsection{Approximate Deletion and Differential Privacy}

We recall the definition of differential privacy \cite{dwork2014algorithmic}. A map, $A$, from a dataset $D$ to a set of outputs, $\mathcal{H}$, is said to be $\epsilon$-differentially private if, for any two datasets $D_1, D_2$ that differ in a single datapoint, and any subset $S \subset \mathcal{H}$, $\mathbf{Pr}[A(D_1) \in S] \le e^{\epsilon} \cdot \mathbf{Pr}[A(D_2) \in S].$ 

Under the relaxed notion of data deletion, it is natural to consider privatization as a manner to support approximation deletion. The idea could be to privatize the model, and then resolve deletion requests by ignoring them. However, there are some nuances involved here that one should be careful of. For example, differential privacy does not privatize the number of datapoints, but this should not be leaked in data deletion.  Furthermore, since we wish to support a stream of deletions in the online setting, we would need to use \emph{group differential privacy} \cite{dwork2014algorithmic}, which can greatly increase the amount of noise needed for privatization. Even worse, this requires selecting the group size (i.e. total privacy budget) during training time (at least for canonical constructions such as the Laplace mechanism). In differential privacy, this group size is not necessarily a hidden parameter. In the context of deletion, it could leak information about the total dataset size as well as how many deletions any given model instance has processed. While privatization-like methods are perhaps a viable approach to support approximate deletion, there remain some technical details to work out, and this is left for future work.

\section{Algorithmic Details}

In Appendix B, we present psuedo-code for the algorithms described in Section 3. We also reference \texttt{https://github.com/tginart/deletion-efficient-kmeans} for Python implementations of our algorithms.

\subsection{Quantized \(k\)-Means}

We present the psuedo-code for Q-$k$-means (Algo. 4). Q-$k$-means follows the iterative protocol as the canonical Lloyd's (and makes use of the $k$-means++ initialization). As mentioned in the main body, there are four key variations from the canonical Lloyd's algorithm that make this method different: quantization, memoization, balance correction, and early termination. The memoization of the optimization state and the early termination for increasing loss are self-explanatory from Algo. 4. We provide more details concerning the quantization step and the balance correction in Appendix B.1.1 and B.1.2 respectively.

    \begin{minipage}{.99\linewidth}
\begin{algorithm}[H]
   \caption{Quantized $k$-means}
   \label{qkmeans_full}
\begin{algorithmic}
\footnotesize
   \STATE {\bfseries Input:} data matrix $D \in \mathbf{R}^{n \times d}$
   \STATE {\bfseries Parameters:} $k \in \mathbf{N}$, $T \in \mathbf{N}$, $\gamma \in (0,1)$, $\epsilon > 0$
   \STATE $c \gets k^{++}(D)$ // \textit{initialize centroids with $k$-means++}
   \STATE Save initial centroids:  $\textsf{save}(c)$.
   \STATE $L \gets k$-means loss of initial partition $\pi$
   \FOR{$\tau = 1$ {\bfseries to} $T$}
        \STATE Store current centroids: $ c'\gets c$
        \STATE Compute centroids: $ c \gets c(\pi)$
       \FOR {$\kappa = 1$ {\bfseries to} $k$}
        \IF{$|\pi(c_\kappa)| < \frac{\gamma n}{k} $ } 
            \STATE Apply correction to $\gamma$-imbalanced partition:
            \STATE  $c_\kappa \gets |\pi(c_\kappa)|c_\kappa + (\frac{\gamma n}{k}-|\pi(c_\kappa)|)c'_\kappa $
        \ENDIF
       \ENDFOR
       \STATE Generate random phase $\theta \sim \mathbf{Unif}[-\frac{1}{2},\frac{1}{2}]^d$
        \STATE Quantize to $\epsilon$-lattice: $\hat{c} \gets Q(c;\theta)$
       \STATE Update partition: $\pi' \gets \pi(\hat{c})$
      \STATE Save state to metadata: $\textsf{save}(c,\theta,\hat{c},|\pi'|)$
       \STATE Compute loss $L'$
       \IF{$L' < L$}
        \STATE $(c, \pi, L) \gets (\hat{c}, \pi' , L')$ //update state
        \ELSE
        \STATE \textbf{break}
       \ENDIF
   \ENDFOR
   \\
   \textbf{return} $c$ //output final centroids as model
\end{algorithmic}
\end{algorithm}
\end{minipage}

Although it is rare, it is possible for a Lloyd's iteration to result in a degenerate (empty) cluster. In this scenario, we have two reasonable options. All of the theoretical guarantees are remain valid under both of the following options. The first option is to re-initialize a new cluster via a $k$-means++ seeding. Since the number of clusters $k$ and iterations $T$ are constant,  this does not impact any of the asymptotic deletion efficiency results. The second option is to simply leave a degenerate partition. This does not impact the upper bound on expected statistical performance which is derived only as a function of the $k$-means++ initialization. For most datasets, this issue hardly matters in practice, since Lloyd's iterations do not usually produce degenerate clusters (even the presence of quantization noise). In our implementation, we have chosen to re-initialize degenerate clusters, and are careful to account for this in our metadata, since such a re-initialization could trigger the need to retrain at deletion time if the re-initialized point is part of the deletion stream.

We present the pseudo-code for the deletion operation (Algo. 5), and then elaborate on the quantization scheme and balance correction.

    \begin{minipage}{.99\linewidth}
\begin{algorithm}[H]
  \caption{Deletion Op for Q-$k$-means }
  \label{delopq}
\begin{algorithmic}
\small
  \STATE {\bfseries Input:} data matrix $D \in \mathbf{R}^{n \times d}$,  target deletion index $i$, training metadata 
  \STATE Obtain target deletion point $p \gets D_i$
  \STATE Retrieve initial centroids from metadata: $\textsf{load}(c_0)$ 
     \IF{$p\in  c_0$ } 
            \STATE // Selected initial point.
            \STATE \textbf{return} $\textsf{Q-k-means}(D_{-i})$ // Need to retrain from scratch.
        \ELSE
    \FOR{$\tau = 1$ {\bfseries to} $T$}
    \STATE Retrieve state for iteration $\tau$: $\textsf{load}(c,\theta,\hat{c},|\pi|)$
     \STATE Cluster assignment of $p$: $\kappa\gets \mathbf{argmin}_k|p-c_k|$
     \STATE Perturbed centroid: $c'_\kappa \gets c_\kappa- p/|\pi(c_{\kappa})|$ 
     \STATE Apply $\gamma$-correction to $c'_\kappa$ if necessary
    \STATE Quantize perturbed centroid: $\hat{c}'_\kappa \gets Q(c'_\kappa;\theta)$
     \IF{$\hat{c}'_\kappa  \neq \hat{c}_\kappa$ }
            \STATE // Centroid perturbed $\rightarrow$ unstable quantization
             \STATE \textbf{return} $\textsf{Q-k-means}(D_{-i})$ // Need to retrain from scratch.
    \ENDIF
        \STATE Update metadata with perturbed state: $\textsf{save}(c',\theta,\hat{c}',|\pi'|)$
   \ENDFOR
   \ENDIF
    \STATE Update $D \gets D_{-i}$\\
  \textbf{return} $\hat{c}$ //Successfully verified centroid stability
\end{algorithmic}
\end{algorithm}
\end{minipage}

We proceed elaborate on the details of the balance correction and quantization steps 

\subsubsection{$\gamma$-Balanced Clusters}

\begin{definition}{$\gamma$-Balanced }

Given a partition $\pi$, we say it us \textit{$\gamma$-balanced} if $|\pi_{\kappa}| \geq \frac{\gamma n}{k}$ for all partitions $\kappa$. The partition is $\gamma$-imbalanced if it is not $\gamma$-balanced. 
\end{definition}

In Q-$k$-means, imbalanced partitions can lead to unstable quantized centroids. Hence, it is preferable to avoid such partitions. As can be seen in the pseudo-code, we add mass to small clusters to correct for $\gamma$-unbalance. At each iteration we apply the following formula on all clusters such that $|\pi_{\kappa}| \geq \frac{\gamma n}{k}$: 
$c_\kappa \gets |\pi(c_\kappa)|c_\kappa + (\frac{\gamma n}{k}-|\pi(c_\kappa)|)c'_\kappa $ where $c'$ denotes the centroids from the previous iteration. 

In prose, for small clusters, current centroids are averaged with the centroids from the previous iteration to increase stability.

For use in practice, a choice of $\gamma$ must be made. If no class balance information is known, then, based on our observations, setting $\gamma = 0.2$ is a solid heuristic for all but severely imbalanced datasets, in which case it is likely that DC-$k$-means would be preferable to Q-$k$-means.
\subsubsection{Quantizing with an $\epsilon$-Lattice}

We detail the quantization scheme used. A quantization maps analog values to a discrete set of points. In our scheme, we uniformly cover $\mathbf{R}^d$ with an $\epsilon$-lattice, and round analog values to the nearest vertex on the lattice. It is also important to add an independent, uniform random phase shift to each dimension of lattice, effectively de-biasing the quantization.

We proceed to formally define our quantization $Q_{(\epsilon,\theta)}$. $Q_{(\epsilon,\theta)}$ is parameterized by a phase shift $\theta \in [-\frac{1}{2},\frac{1}{2}]^d$ and an granularity parameter $\epsilon > 0$. For brevity, we omit the explicit dependence on phase and granularity when it is clear from context. For a given $(\epsilon,\theta)$:

$$ a(x) = \mathbf{argmin}_{{j} \in \mathbf{Z}^d} \{||x - \epsilon(\theta + {j})||_2\}$$
$$ Q(x) = \epsilon(\theta + a(x)) $$ 

We set $\{\theta\}_0^t$ with an iid random sequence such that $\theta_\tau \sim \mathbf{Unif}[-\frac{1}{2},\frac{1}{2}]^d$.

\subsection{Divide-and-Conquer $k$-Means}

We present pseudo-code for the deletion operation of divide-and-conquer $k$-means. The pseudo-code for the training algorithm may be found in the main body. The deletion update is conceptually simple. Since a deleted datapoint only belong to one leaf's dataset, we only need recompute the sub-problems on the path from said leaf to the root. 

\begin{algorithm}[h!]
   \caption{Deletion Op for DC-$k$-means}
   \label{delopdck}
\begin{algorithmic}
\footnotesize
  \STATE {\bfseries Input:} data matrix $D \in \mathbf{R}^{n \times d}$,  target deletion index $i$, model metadata $M$
  \STATE Obtain target deletion point $p \gets D_i$
  \STATE $\textsf{node} \gets$  leaf node assignment of $p$
  \STATE $\textsf{node}.\textsf{dataset} \gets \textsf{node}.\textsf{dataset}\setminus p$

       \WHILE{$\textsf{node}$ is not root}

    \STATE $ \textsf{node.parent}.\textsf{data} \gets \textsf{node.parent}.\textsf{data} \setminus \textsf{node}.\textsf{centroids}$
    \STATE $\textsf{node}.\textsf{centroids} \gets \textsf{k-means++}(\textsf{node}.\textsf{data},k,T)$
    \STATE $\textsf{node.parent}.\textsf{dataset}.\textsf{add}(\textsf{node}.\textsf{centroids})$
    \STATE $\textsf{node} \gets \textsf{node.parent}$
    \ENDWHILE
\STATE $\textsf{node}.\textsf{centroids} \gets \textsf{k-means++}(\textsf{node}.\textsf{dataset},k,T)$
  \STATE Update $D \gets D_{-i}$
  
  \textbf{return} \textsf{node.centroids}

\end{algorithmic}
\end{algorithm}

\subsection{Initialization for $k$-Means}
For both of our algorithms, we make use of the $k$-means++ initialization scheme \cite{arthur2007k}. This initialization is commonplace, and is the standard initialization in many scientific libraries, such as Sklearn \cite{scikit-learn}. In order to provide a more self-contained presentation, we provide some pseudo-code for the $k$-means++ initialization.

\begin{algorithm}[h!]
   \caption{Initialization by $k$-means++}
   \label{init++}
\begin{algorithmic}
\footnotesize
  \STATE {\bfseries Input:} data matrix $D \in \mathbf{R}^{n \times d}$,  number of clusters $k$
  \STATE $i \gets \mathbf{Uni}\{1,...,n\}$
    \STATE $I \gets \{D_{i}\}$.
    \STATE $\mathbf{u} \gets 0^n \in \mathbf{R}^n$.
       \FOR{$1 < l \leq k$}
        \STATE   $ u_j = \min_{\eta \in I}||\eta - D_j||^2$ for all $1 \leq j \leq n$ 
        \STATE $Z = \sum_{j=1}^{n}u_j$
        \STATE Sample $i \sim \frac{1}{Z}\mathbf{u}$
        \STATE $I \gets I \cup \{D_i\}$
    \ENDFOR
  
  \textbf{return} $I$
\end{algorithmic}
\end{algorithm}

\section{Mathematical Details}

Here, we provide proofs for the claims in the main body. We follow notation introduced in the main body and Appendix A. As a notational shorthand, we will let denote $A(D_{-i})$ by $A_{-i}$ and $R(D,A(D),D_{-i})$ as $R$ when there is only one dataset in the context. Also, when it is unambiguous, we will use $A$ to denote the \emph{specific} learning algorithm in question, and $R$ to denote its corresponding deletion operation. 

\subsection{Proof of Theorem 4.1}

 Refer to the main body for the statement of Theorem 4.1. Here is an abridged version:
 
\begin{theorem*}
Q-$k$-means supports $m$ deletions in expected time $O(m^2d^{5/2}/\epsilon)$.
\end{theorem*}

 Note that we assume the dataset is scaled onto the unit hypercube. Otherwise, the theorem still holds with an assumed constant factor radial bound. We prove the theorem in three successive steps, given by Lemma C.1 through Lemma C.3.

\begin{lemma}
Define $C = [-\frac{\epsilon}{2},\frac{\epsilon}{2}]^d$ for some $\epsilon > 0$. $C$ is the hypercube in Euclidean $d$-space of side length $\epsilon$ centered at the origin. Let $C' = [\frac{-\epsilon + \epsilon'}{2},\frac{\epsilon - \epsilon'}{2}]^d$ for some $\epsilon' < \epsilon$. Let $X$ be a uniform random variable with support $C$. Then, $\mathbf{Pr}[X \in C/C'] \leq \frac{2d\epsilon'}{\epsilon}$. 
\end{lemma}
\begin{proof}{\textit{(Lemma C.1)}}
If $X \in C/C'$, then there exists some $i \in \{1,...,d\}$ such that $X_i \in [-\frac{\epsilon}{2},\frac{-\epsilon + \epsilon'}{2}] \cup [\frac{\epsilon - \epsilon'}{2},\frac{\epsilon}{2} ]$. Marginally, $\mathbf{Pr}[X_i \in [-\frac{\epsilon}{2},\frac{-\epsilon + \epsilon'}{2}] \cup [\frac{\epsilon - \epsilon'}{2},\frac{\epsilon}{2} ]] = 2\epsilon'/\epsilon.$ Taking a union bound over the $d$ dimensions obtains the bound. 
\end{proof}
We make use of Lemma C.1 in proving the following lemma. First, recall the definition of our quantization scheme $Q$ from Section 3:
$$Q(x) = \epsilon(\theta + \mathbf{argmin}_{{j} \in \mathbf{Z}^d} \{||x - \epsilon(\theta + {j})||_2\}).$$ We take $\theta \sim \mathbf{Uni}[-\frac{1}{2},\frac{1}{2}]^d$, implying a distribution for $Q$.

\begin{lemma}
Let $Q$ be a uniform quantization $\epsilon$-lattice over $\mathbf{R}^d$ with uniform phase shift $\theta$.  Let $Q(\cdot)$ denote the quantization mapping over $\mathbf{R}^d$ and let $Q[\cdot]$ denote the quantization image for subsets of $\mathbf{R}^d$. Let $X \in \mathbf{R}^d$. Then $\mathbf{Pr}[\{Q(X)\} \neq Q[B_{\epsilon'}(X)]] < \frac{2d\epsilon'}{\epsilon}$, where $B_{\epsilon'}(X)$ is the $\epsilon'$-ball about $X$ under Euclidean norm.
\end{lemma}
\begin{proof}{\textit{(Lemma C.2)}}
Due to invariance of measure under translation, we may apply a coordinate transformation by translation $Q(X)$ to the origin of $\mathbf{R}^d$. Under this coordinate transform, $X \sim \mathbf{Uni}[-\frac{\epsilon}{2},\frac{\epsilon}{2}]^d$. Further, note that $\mathbf{Pr}[B_{\epsilon'}(X) \subset [-\frac{\epsilon}{2},\frac{\epsilon}{2}]^d]$ is precisely equivalent to $\mathbf{Pr}[X \in  [\frac{-\epsilon + \epsilon'}{2},\frac{\epsilon - \epsilon'}{2}]^d]$. Because $X$ is uniform, applying Lemma C.1 as an upper bound completes the proof.
\end{proof}

With Lemma C.2 in hand, we proceed to state and prove Lemma C.3.

\begin{lemma}
Let $D$ be an dataset on $[0,1]^d$ of size $n$. Let $\mathbf{\overline{c}}{(D)}$ be the centroids computed by Q-$k$-means with initialization $I$ and parameters $T$, $k$, $\epsilon$, and $\gamma$. Then, with  probability greater than $1 - \frac{2mT kd^{3/2}}{\epsilon \gamma n}$, it holds that $\mathbf{\overline{c}}{(D)} = \mathbf{\overline{c}}{(D_{-\Delta})}$ for any $\Delta \subset D$ with $|\Delta| \leq m$ and $\Delta \cap I = \emptyset$, where probability is with respect to the randomness in the quantization phase.
\end{lemma}

\begin{proof}{\emph{(Lemma C.3)}}
 We analyze two instances of Q-$k$-means algorithm operating with the same initial centroids and the same sequence of iid quantization maps $\{Q_\tau \}_1^t$. One instance runs on input dataset $D$ and the other runs on input dataset $D_{-\Delta}$. This is the only difference between the two instances.
 
Let $\mathbf{c}_I^{(\tau,\kappa)}(\delta)$ denote the $\kappa$-th analog (i.e. non-quantized) centroid at the $\tau$-th iteration of Q-$k$-means on some input dataset $\delta$ with initialization $I$. By construction, for any datasets $\delta, \delta'$, we have that $\mathbf{\overline{c}}_I(\delta) = \mathbf{\overline{c}}_I(\delta')$ if $Q_\tau(\mathbf{c}_I^{(\tau,\kappa)}(\delta))$ = $Q_\tau(\mathbf{c}_I^{(\tau,\kappa)}(\delta'))$ for all $\tau \in \{1,...,t\}$ and all $\kappa \in \{1,...,k\}$.
 
Fix any particular $\tau$ and $\kappa$. We can bound $||\mathbf{c}_I^{(\tau,\kappa)}(D) - \mathbf{c}_I^{(\tau,\kappa)}(D_{-\Delta})||_2$ as follows. Note that $\mathbf{c}_I^{(\tau,\kappa)}(D) = \frac{1}{n}\sum_{i=1}^n D_i \mathbf{1}(D_i \in \pi_\kappa)$ where $\mathbf{1}(\cdot)$ denotes the indicator function. Furthermore, $\mathbf{c}_I^{(\tau,\kappa)}(D_{-\Delta}) = \frac{1}{n}\sum_{i=1}^n D_i \mathbf{1}(D_i \in \pi_\kappa) \mathbf{1}(D_i \notin \Delta)$. Assume that $|\pi_\kappa| \geq \gamma n/k$. Because $|\Delta| \leq m$ and $||D_i||_2 \leq \sqrt{d}$, these sums can differ by at most $\frac{m k\sqrt{d}}{ \gamma n}$. On the other hand, assume that $|\pi_\kappa| < \gamma n/k$. In this case, the $\gamma$-correction still ensures that the sums differ by at most $\frac{m k\sqrt{d}}{\gamma n}$. This bounds $||\mathbf{c}_I^{(\tau,\kappa)}(D) - \mathbf{c}_I^{(\tau,\kappa)}(D-\Delta)||_2 \leq \frac{m k \sqrt{d}}{\gamma n}$.

To complete the proof, apply Lemma C.2 setting $\epsilon'= \frac{m k \sqrt{d}}{\gamma n}$. Taking a union bound over $\tau \in \{1,...,t\}$ and $\kappa \in \{1,...,k\}$ yields the desired result.

\end{proof}

We are now ready to complete the theorem. We briefly sketch and summarize the argument before presenting the proof. Recall the deletion algorithm for Q-$k$-means (Appendix B). Using the runtime memo, we verify that the deletion of a point does not change what would have been the algorithm's output. If it would have changed the output, then we retrain the entire algorithm from scratch. Thus, we take a weighted average of the computation expense in these two scenarios. Recall that retraining from scratch takes time $O(nkTd)$ and verifying the memo at deletion time takes time $O(kTd)$. Finally, note that we must sequentially process a sequence of $m$ deletions, with a valid model output after each request. We are mainly interested in the scaling with respect to $m$, $\epsilon$ and $n$, treating other factors as non-asymptotic constants in our analysis. We now proceed with the proof.

\begin{theorem*}
Q-$k$-means supports $m$ deletions in expected time $O(m^2d^{5/2}/\epsilon)$.
\end{theorem*}

\begin{proof}{\emph{(Correctness)}}\\
In order for $R$ to be a valid deletion, we require that $R =_d A_{-i}$ for any $i$. In this setting, we identify models with the output centroids: $A(D) = \overline{c}_{I}(D)$. Consider the sources of randomness: the iid sequence of random phases and the $k$-means++ initializations. 

Let $I(\cdot)$ be a set-valued random function computing the $k$-means++ initializations over a given dataset. $E$ denote the event that $D_i \notin I(D)$. Then, from the construction of $k$-means++, we have that for all $j \neq i$, $\mathbf{Pr}[D_j \in  I(D_{-i})] = \mathbf{Pr}[D_j \in I(D)| E]$. Thus, $I(D)$ and $I(D_{-i})$ are equal in distribution conditioned on $i$ not being an initial centroid. Note that this is evident from the construction of $k$-means++ (see algorithm 7).

Let $\theta$ denote the iid sequence of random phases for $A$ and let $\theta_{-i}$ denote the iid sequence of random phases for $A_{-i}$.
Within event $E$, we define a set of events $E'(\hat{\theta})$, parameterized by $\hat{\theta}$, as the event that output centroids are stable under deletion conditioned on a given sequence of phases $\hat{\theta}$: 

$E = \{i \notin I(D)\}, E'(\hat{\theta}) = \{A|\{\theta = \hat{\theta}\} = A_{-i}|\{\theta_{-i} = \hat{\theta}\}\} \cap E $

By construction of $R$, we have $R = A\mathbf{1}(E'(\theta)) + A_{-i}\mathbf{1}(\overline{E'(\theta)})$ where event $E'(\theta)$ is verified given the training time memo. To conclude, let $S$ be any Borel set:

$\mathbf{Pr}[R \in S] = \mathbf{Pr}[E'(\theta)]\mathbf{Pr}[R \in S | E'(\theta)] + (1-\mathbf{Pr}[E'(\theta)])\mathbf{Pr}[R \in S | \overline{E'(\theta)}]$ by law of total probability.

$= \mathbf{Pr}[E'(\theta)]\mathbf{Pr}[A \in S | E'(\theta)] + (1-\mathbf{Pr}[E'(\theta)])\mathbf{Pr}[A_{-i} \in S]$ by construction of $R$

$= \mathbf{Pr}[E'(\theta)]\mathbf{Pr}[A_{-i} \in S |\theta_{-i} = \theta] + (1-\mathbf{Pr}[E'(\theta)])\mathbf{Pr}[A_{-i} \in S]$ by definition of $E'$

$= \mathbf{Pr}[E'(\theta)]\mathbf{Pr}[A_{-i} \in S] + (1-\mathbf{Pr}[E'(\theta)])\mathbf{Pr}[A_{-i} \in S] = \mathbf{Pr}[A_{-i} \in S]$ by $\theta =_d \theta_{-i}$.

\end{proof}

\begin{proof}{\emph{(Runtime)}}

Let $\mathcal{T}$ be the total runtime of $R$ after training $A$ once and then satisfying $m$ deletion requests with $R$. Let  $\Delta = \{i_1,i_2,...,i_m\}$ denote the deletion sequence, with each deletion sampled uniformly without replacement from $D$.

Let $\Psi$ be the event that the centroids are stable for all $m$ deletions. Using Theorem 3.1 to bound the event complement probability $\mathbf{Pr}(\overline{\Psi})$:

$\mathbf{E}\mathcal{T} \leq \mathbf{E}[\mathcal{T}|\Psi] + \mathbf{Pr}[\overline{\Psi}]\mathbf{E}[\mathcal{T}|\overline{\Psi}] = O(mkTd) + O(\epsilon^{-1} m^2T^2k^3d^{2.5}) = O(m^2d^{2.5}/\epsilon)$.

In $\Psi$ the centroids are stable, and verifying in $\Psi$ takes time $O(mktd)$ in total. In $\overline{\Psi}$ we coarsely upper bounded $\mathcal{T}$ by assuming we re-train to satisfy each deletion. 
\end{proof}

\subsection{Proofs of Corollaries and Propositions}

We present the proofs of the corollaries in the main body. We are primarily interested in the asymptotic effects of $n$, $m$, $\epsilon$, and  $w$. We treat other variables as constants. For the purposes of online analysis, we let $\epsilon = \Theta(n^{-\beta})$ for some $\beta \in (0,1)$ and $w = \Theta(n^\rho)$ for some $\rho \in (0,1)$

\subsubsection{Proof of Corollory 4.1.1}
We state the following Theorem of Arthur and Vassilvitskii concerning $k$-means++ initializations \cite{arthur2007k}:

\begin{theorem}{(Arthur and Vassilivitskii)}

Let $\mathcal{L}^*$ be the optimal loss for a $k$-means clustering problem instance. Then $k$-means++ achieves expected loss $\mathbf{E}\mathcal{L}^{++} \leq (8\ln k + 16)\mathcal{L}^*$

\end{theorem}

We re-state corollary 4.1.1:

\begin{corollary*}
Let $\mathcal{L}$ be a random variable denoting the loss of Q-$k$-means on a particular problem instance of size $n$. Then $\mathbf{E}\mathcal{L} \leq (8 \ln k + 16)\mathcal{L}^{*} + \epsilon\sqrt{nd(8 \ln k + 16)\mathcal{L}^{*}} + \frac{1}{4}nd\epsilon^2$.
\end{corollary*}

\begin{proof}
Let $c$ be the initialization produced $k$-means++. Let $\mathcal{L}^{++} = \sum_{i=1}^{n}|| c(i) - x_i ||_2^2$ where $c(i)$ is the centroid closest to the $i$-th datapoint. Let $|| c(i) - x_i ||_2^2 = \sum_{j=1}^d \delta_{ij}^2$, with $\delta_{ij}$ denoting the scalar distance between $x_i$ and $c(i)$ in the $j$-th dimension. Then we may upper bound $\mathcal{L} \leq \sum_{i=1}^{n} \sum_{j=1}^d (\delta_{ij}^2 +  \frac{1}{4}\epsilon^2 + \delta_{ij}\epsilon)$ by adding a worst-case $\frac{\epsilon}{2}$ quantization penalty in each dimension. This sum reduces to:

$\mathcal{L} \leq \sum_i^n \sum_j^d \delta_{ij}^2  + \sum_i^n \sum_j^d \frac{1}{4}\epsilon^2 + \sum_i^n \sum_j^d  \delta_{ij}\epsilon  = \mathcal{L}^{++} + \frac{1}{4}nd\epsilon^2 + \epsilon \sqrt{nd\mathcal{L}^{++}}$. The third term comes from the fact that $\sqrt{nd\mathcal{L}^{++}} \geq \sum_i^n \sum_j^d  \delta_{ij} \geq \sqrt{\mathcal{L}^{++}}$ if $\sum_i^n \sum_j^d  \delta_{ij}^2 = \mathcal{L}^{++}$ and $\delta_{ij} > 0$ (to see this, treat it as a constrained optimization over the $\delta_{ij}$). Thus:

$$ \mathbf{E}\mathcal{L} \leq \mathbf{E}\mathcal{L}^{++} +  \frac{1}{4}nd\epsilon^2 + \epsilon\sqrt{nd} \mathbf{E}\sqrt{\mathcal{L}^{++}} $$

Using Jensen's inequality \cite{cover2012elements} yields $\mathbf{E}\sqrt{\mathcal{L}} \leq \sqrt{\mathbf{E}\mathcal{L}}$:

$$ \mathbf{E}\mathcal{L} \leq \mathbf{E}\mathcal{L}^{++} +  \frac{1}{4}nd\epsilon^2 + \epsilon\sqrt{nd} \sqrt{\mathbf{E}\mathcal{L}^{++}} $$

To complete the proof, apply Theorem C.5:
$$\mathbf{E}\mathcal{L} \leq C\mathcal{L}^* +  \frac{1}{4}nd\epsilon^2 + \epsilon\sqrt{ndC\mathcal{L}^*}$$
where $C = 8 \ln k + 16$. 
\end{proof}

\subsubsection{Proof of Proposition 4.2}

\begin{proposition*}
Let $D$ be an dataset on $\mathbf{R}^d$ of size $n$. Fix parameters $T$ and $k$ for DC-$k$-means. Let $w = \Theta(n^\rho)$ and $\rho \in (0,1)$  Then, with a depth-1, $w$-ary divide-and-conquer tree, DC-$k$-means supports $m$ deletions in time $O(mn^{\mathbf{max}(\rho,1-\rho)}d)$ in expectation with probability over the randomness in dataset partitioning.
\end{proposition*}

\begin{proof}\emph{(Correctness)}
We require that $R(D,A(D),i) =_d A(D_{-i})$. Since each datapoint is assigned to a leaf independently, the removal of a datapoint does not change the distribution of the remaining datapoints to leaves. However, one must be careful when it comes to the number of leaves, which cannot change due to a deletion. This is problematic if the number of leaves is $\left \lceil{n^\rho}\right \rceil $ (or another similar quantity based on $n$). 

The simplest way to address this (without any impact on asymptotic rates) is to round the number of leaves to the nearest power of 2. This works because the intended number of leaves will only be off from $n^\rho$ by at most a factor of 2. In the rare event this rounding changes due to a deletion, we will have to default to retraining from scratch, but, asymptotically in the fractional power regime, this can only happen a constant number of times which does not affect an amortized or average-case time complexity analysis.
\end{proof}

We proceed to prove the runtime analysis.

\begin{proof}\emph{(Runtime)}

Let $\mathcal{T}$ be the total runtime of $R$ after training $A$ once and then satisfying $m$ deletion requests with $R$. Let  $\Delta = \{i_1,i_2,...,i_m\}$ denote the deletion sequence, with each deletion sampled uniformly without replacement from $D$.

Let $S$ be the uniform distribution over $n^{\rho}$ elements and let $\hat{S}$ be the empirical distribution of $n$ independent samples of $S$. The fraction of datapoints assigned to the $i$-th leaf is then modeled by $\hat{S}_i$. We treat $\hat{S}$ as probability vector. Let random variable $J = n\hat{S}_i$ with probability $\hat{S}_i$. Thus, $J$ models the distribution over sub-problem sizes for a randomly selected datapoint. Direct calculation yields the following upper bound on runtime:

$\mathcal{T} \leq m(O(kTdJ) + O(n^\rho k^2Td))$ where the first term is due to the total deletion time at the leaves, the second term is due to the total deletion time at the root, and the $m$ factor is due to the number of deletions.

Hence, we have $\mathbf{E}(\mathcal{T}) \leq O(mkTd)\mathbf{E}(J) + O(mn^\rho k^2Td)$, with $\mathbf{E}(J)$ representing the quantity of interest. Computing $\mathbf{E}(J)$ is simple using the second moments of the Binomial distribution , denoted by $\mathcal{B}$:

$$\mathbf{E}(J) = \mathbf{E}(\mathbf{E}(J|\hat{S})) = \sum_{i=1}^{n^\rho}n\mathbf{E}(\hat{S}_i^2) $$

Noting that $\hat{S}_i \sim \frac{1}{n}\mathcal{B}(n,n^{-\rho})$  and $\mathbf{E}((\mathcal{B}(n,p)^2) = n(n-1)p + np$ \cite{knoblauch2008closed} yields:

$$ = n^{\rho-1}\mathbf{E}((\mathcal{B}(n,n^{-\rho})^2) = O(n^{1-\rho})$$

This yields the final bound: $\mathbf{E}(\mathcal{T}) \leq O(n^{1-\rho}mkTd) + O(n^\rho mk^2Td) = O(m\mathbf{max}\{n^{1-\rho}, n^\rho \}d)$

\end{proof}

\subsubsection{Proof of Corollary 4.2.1}

\begin{corollary*}

With $\epsilon = \Theta(n^{-\beta})$ for $0 < \beta < 1$, Q-$k$-means algorithm is deletion efficient in expectation if $\alpha \leq \frac{1-\beta}{2}$

\end{corollary*}

\begin{proof}

We are interested in the asymptotic scaling of $n$, $m$, and $\epsilon$, and treat other factors as cosntants. We begin with the expected deletion time from Theorem 4.1, given by $O(m^2d^{5/2}\epsilon^{-1})$. Recall we are using rates $\epsilon = \Theta(n^{-\beta})$ and $m = \Theta(n^{\alpha})$. Applying the rates, adding in the training time, and amortizing yields $O(n^{1-\alpha} + n^{\alpha+\beta})$. Thus, deletion efficiency follows if $1-\alpha > \alpha + \beta$. Rearranging terms completes the calculation.
\end{proof}

\subsubsection{Proof of Corollary 4.2.2}

\begin{corollary*}

With $w = \Theta(n^{\rho})$ and a depth-1 $w$-ary divide-and-conquer tree, DC-$k$-means is deletion efficient in expectation if $\alpha \leq 1- \textbf{max}(\rho,1-\rho)$
\end{corollary*}

\begin{proof}

We are interested in the asymptotic scaling of $n$, $m$, and $w$, and treat other factors as constants. Recall we are using rates $w = \Theta(n^{\rho})$ and $m = \Theta(n^{\alpha})$ By Proposition 4.2, the runtime of each deletion is upper bounded by $O(n^{\mathbf{max}(\rho,1-\rho)})$ and the training time is $O(n)$. Amortizing and comparing the terms yields the desired inequaltiy. Deletion efficiency follows if $\mathbf{max}\{\rho,1-\rho\} \leq 1 - \alpha$. Rearranging terms completes the calculation.
\end{proof}

\section{Implementation and Experimental Details}
We elaborate on implementation details, the experimental protocol used in the main body, and present some supplementary experiments that inform our understanding of the proposed techniques.
\subsection{Experimental Protocol}

 We run a $k$-means baseline (i.e. a $k$-means++ seeding followed by Lloyd's algorithm), Q-$k$-means, and DC-$k$-means on 6 datasets in the simulated online deletion setting. As a proxy for deletion efficiency, we report the wall-clock time of the program execution on a single-core of an Intel Xeon E5-2640v4 (2.4GHz) machine. We are careful to only clock the time used by the algorithm and pause the clock when executing test-bench infrastructure operations. We do not account for random OS-level interruptions such as context switches, but we are careful to allocate at most one job per core and we maintain high CPU-utilization throughout.
 
 For each of the three methods and six datasets, we run five replicates of each benchmark to obtain standard deviation estimates. To initialize the benchmark, each algorithm trains on the complete dataset, which is timed by wall-clock. We then evaluate the loss and clustering performance of the centroids (untimed). Then, each model must sequentially satisfy a sequence of 1,000 uniformly random (without replacement) deletion requests. The time it takes to satisfy each request is also timed and added to the training time to compute a total computation time. The total computation time of the benchmark is then amortized by dividing by 1,000 (the number of deletion requests). This produces a final amortized wall-clock time. For the $k$-means baseline, we satisfy deletion via naive re-training. For Q-$k$-means and DC-$k$-means we use the respective deletion operations. As part of our benchmark, we also evaluate the statistical performance of each method after deletions 1,10,100, and 1,000. Since we are deleting less than 10\% of any of our datasets, the statistical performance metrics do not change significantly throughout the benchmark and neither do the training times (when done from scratch). However, a deletion operation running in time significantly less than it takes to train from scratch should greatly reduce the total runtime of the benchmark. Ideally, this can be achieved without sacrificing too much cluster quality, as we show in our results (Section 5).

\subsubsection{Implementation Framework}

We are interested in \emph{fundamental} deletion efficiency, however, empirical runtimes will always be largely implementation specific. In order to minimize the implementation dependence of our results, we control by implementing an in-house version of Lloyd's iterations which is used as the primary optimization sub-routine in all three methods. Our solver is based on the Numpy Python library \cite{van2011numpy}. Thus, Q-$k$-means and DC-$k$-means use the same sub-routine for computing partitions and centroids as does the $k$-means baseline. Our implementation for all three algorithms can be found at \url{https://github.com/tginart/deletion-efficient-kmeans}.

\subsubsection{Heuristic Parameter Selection}
Hyperparameter tuning poses an issue for deletion efficiency. In order to be compliant to the strictest notions of deletion, we propose the following \emph{heuristics} to select the quantization granularity parameter $\epsilon$ and the number of leaves $w$ for Q-$k$-means and DC-$k$-means, respectively. Recall that we always set iterations to 10 for both methods.

\emph{Heurstic Parameter Selection for Q-$k$-means}. Granularity $\epsilon$ tunes the centroid stability versus the quantization noise. Intuitively, when the number of datapoints in a cluster is high compared to the dimension, we need lower quantization noise to stabilize the centroids. A good rule-of-thumb is to use $ \epsilon = 2^{\lfloor- \log_{10} ( \frac{n}{kd^{3/2}}) - 3\rceil}$, which yields an integer power of 2. The heuristic can be conceptualized as capturing the effective cluster mass per dimension of a dataset. We use an exponent of $1.5$ for $d$, which scales like the stability probability (see Lemmas C.1 - C.3). The balance correction parameter $\gamma$ is always set to 0.2, which should work well for all but the most imbalanced of datasets.

\emph{Heurstic Parameter Selection for DC-$k$-means}. Tree width $w$ tunes the sub-problem size versus the number of sub-problems. Intuitively, it is usually better to have fewer larger sub-problems than many smaller ones. A good rule-of-thumb is to set $w$ to $n^{0.3}$, rounded to the nearest power of two.

\subsubsection{Clustering Performance Metrics}
We evaluate our cluster quality using the silhouette coefficient and normalized mutual information, as mentioned in the main body. To do this evaluation, we used the routines provided in the Scikit-Learn Python library \cite{pedregosa2011scikit}. Because computing the silhouette is expensive, for each instance we randomly sub-sample 10,000 datapoints to compute the score.

\subsubsection{Scaling}

We note that all datasets except MNIST undergo a minmax scaling in order to map them into the unit hypercube (MNIST is already a scaled greyscale image). In our main body, we treat this as a one-time scaling inherit to the dataset itself. In practice, the scaling of a dataset can change due to deletions. However, this is a minor concern (at least for minmax scaling) as only a small number of extremal datapoints affect the scale. Retraining from scratch when these points come up as a deletion request does not impact asymptotic runtime, and has a negligible impact on empirical runtime. Furthermore, we point out that scaling is not necessary for our methods to work. In fact, in datasets where the notion of distance remains coherent across dimensions, one should generally refrain from scaling. Our theory holds equally well in the case of non-scaled data, albeit with an additional constant scaling factor such as a radial bound.

\subsection{Datasets}

\begin{itemize} 
    \item \texttt{Celltypes}~\cite{han2018mapping} consists of $12,009$ single cell RNA sequences from a mixture of $4$ cell types: microglial cells, endothelial cells, fibroblasts, and mesenchymal stem cells. The data was retrieved from the Mouse Cell Atlas and consists of $10$ feature dimensions, reduced from an original $23,433$ dimensions using principal component analysis. Such dimensionality reduction procedures are a common practice in computational biology.
    \item \texttt{Postures}~\cite{gardner2014measuring,gardner20143d} consists of $74,975$ motion capture recordings of users performing $5$ different hand postures with unlabeled markers attached to a left-handed glove. 
    \item \texttt{Covtype}~\cite{blackard1999comparative}  consists of $15,120$ samples of $52$ cartographic variables such as elevation and hillshade shade at various times of day for $7$ forest cover types. 
    \item \texttt{Botnet}~\cite{meidan2018n}  contains statistics summarizing the traffic between different IP addresses for a commercial IoT device (Danmini Doorbell). We aim to distinguish between benign traffic data ($49,548$ instances) and $11$ classes of malicious traffic data from botnet attacks, for a total of $1,018,298$ instances.
   \item \texttt{MNIST}~\cite{lecun1998gradient}  consists of $60,000$ images of isolated, normalized, handwritten digits. The task is to classify each $28\times28$ image into one of the ten classes.
    \item \texttt{Gaussian} consists of $5$ clusters, each generated from $25$-variate Gaussian distribution centered at randomly chosen locations in the unit hypercube. $20,000$ samples are taken from each of the $5$ clusters, for a total of $100,000$ samples. Each Gaussian cluster is spherical with variance of $0.8$.
\end{itemize}

\subsection{Supplementary Experiments}

We include three supplementary experiments. Our first is specific to Q-$k$-means (See Appendix D.3.1), and involves the stability of the quantized centroids against the deletion stream. In our second experiment we explore how the choices of key parameters ($\epsilon$ and $w$) in our proposed algorithms contribute to the statistical performance of the clustering. In our third experiment, we explore how said choices contribute to the deletion efficiency in the online setting.
\subsubsection{Re-training During Deletion Stream for Q-$k$-means}

In this experiment, we explore the stability of the quantized centroids throughout the deletion stream. This is important to understand since it is a fundamental behavior of the Q-$k$-means, and is not an implementation or hardware specific as a quantity like wall-clock time is. We plot, as a function of deletion request, the average number of times that Q-$k$-means was forced to re-train from scratch to satsify a deletion request.

\begin{figure}[h!]
\centering
\small
\includegraphics[width=\textwidth]{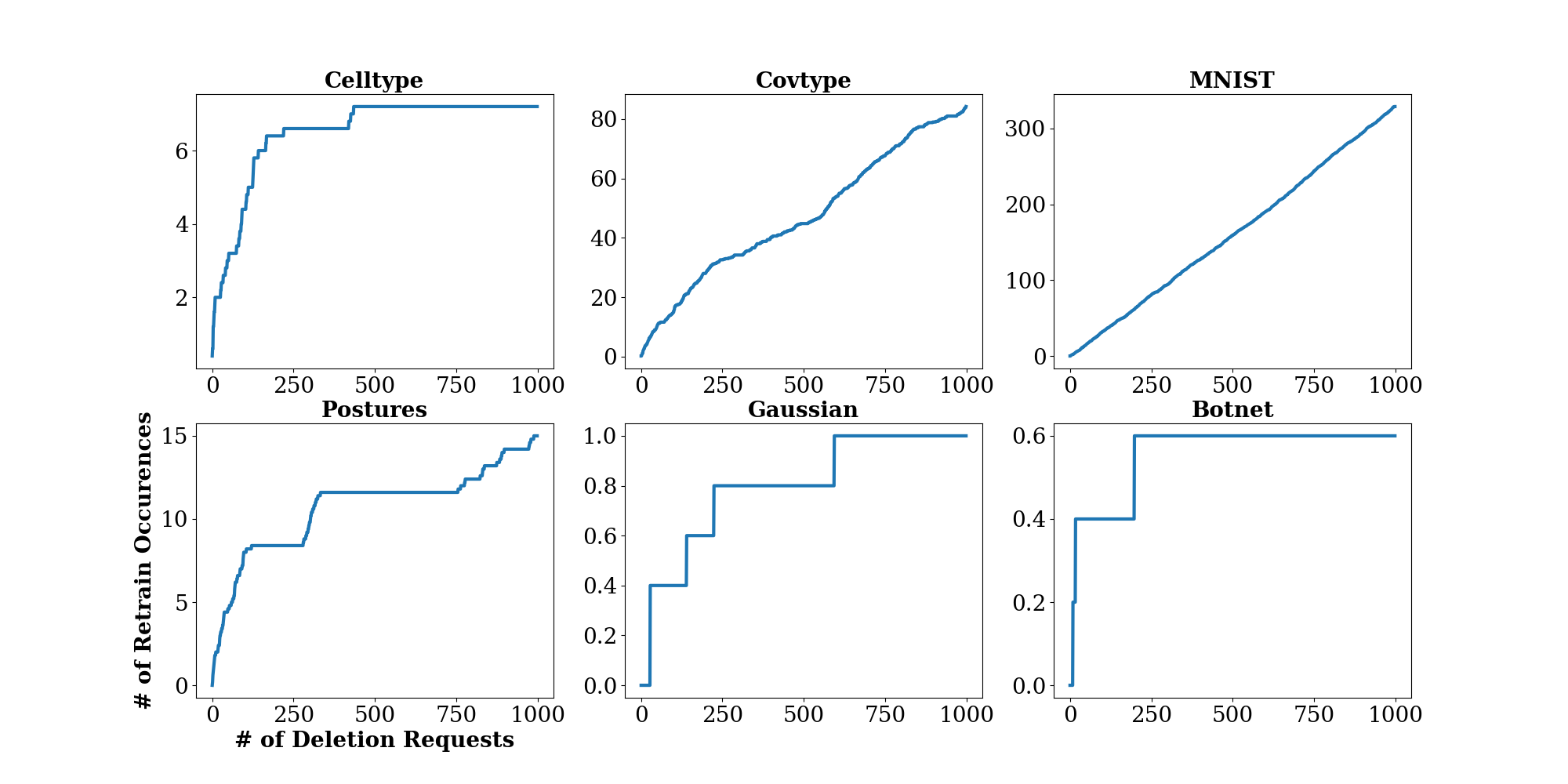}
\caption{Average retrain occurrences during deletion stream for Q-$k$-means}
\end{figure}

As we can see in Fig. 2, when the effective dimensionality is higher (relative to sample size), like in the case of \texttt{MNIST}, our retraining looks like somewhat of a constant slope across the deletions, indicating that the quantization is unable to stabilize the centroids for an extended number of deletion requests.

\subsubsection{Effects of Quantization Granularity and Tree Width on Optimization Loss}

Although the viability of hyperparameter tuning in the context of deletion efficient learning is dubious, from a pedagogical point of view, it is still interesting to sweep the main parameters (quantization granularity $\epsilon$ and tree width $w$) for the two proposed methods.
In this experiment, we compare the $k$-means optimization loss for a range of $\epsilon$ and $w$. As in the main body, we normalize the $k$-means objective loss to the baseline and restrict ourselves to depth-1 trees.

\begin{figure}[h!]
\centering
  \centering
  \includegraphics[width=1.05\textwidth]{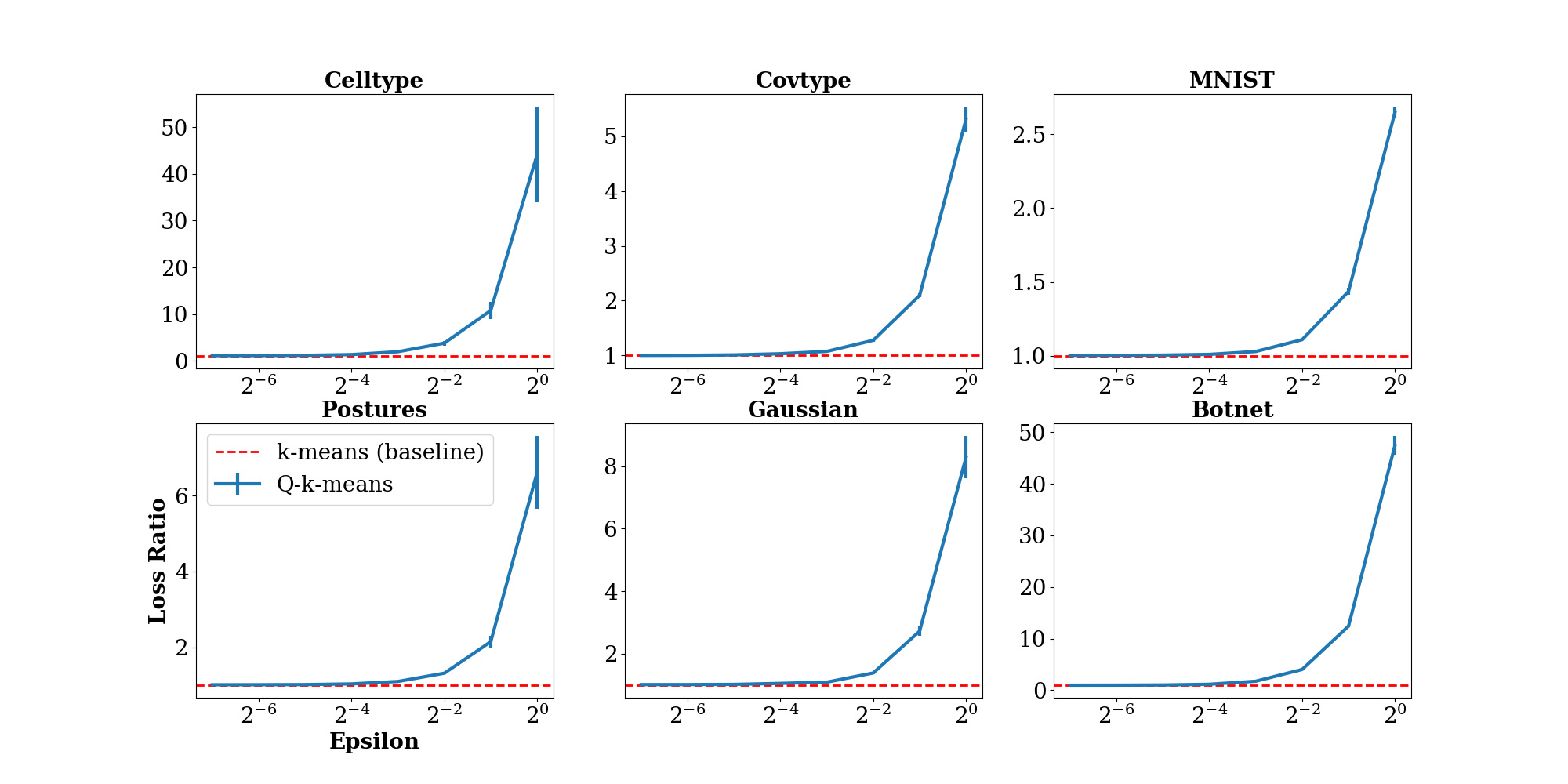}
    \caption{Loss Ratio vs. $\epsilon$ for Q-$k$-means on 6 datasets}
  \label{fig:eps_vs_acc}
\end{figure}

\begin{figure}[h!]
\centering
  \centering
  \includegraphics[width=1.05\textwidth]{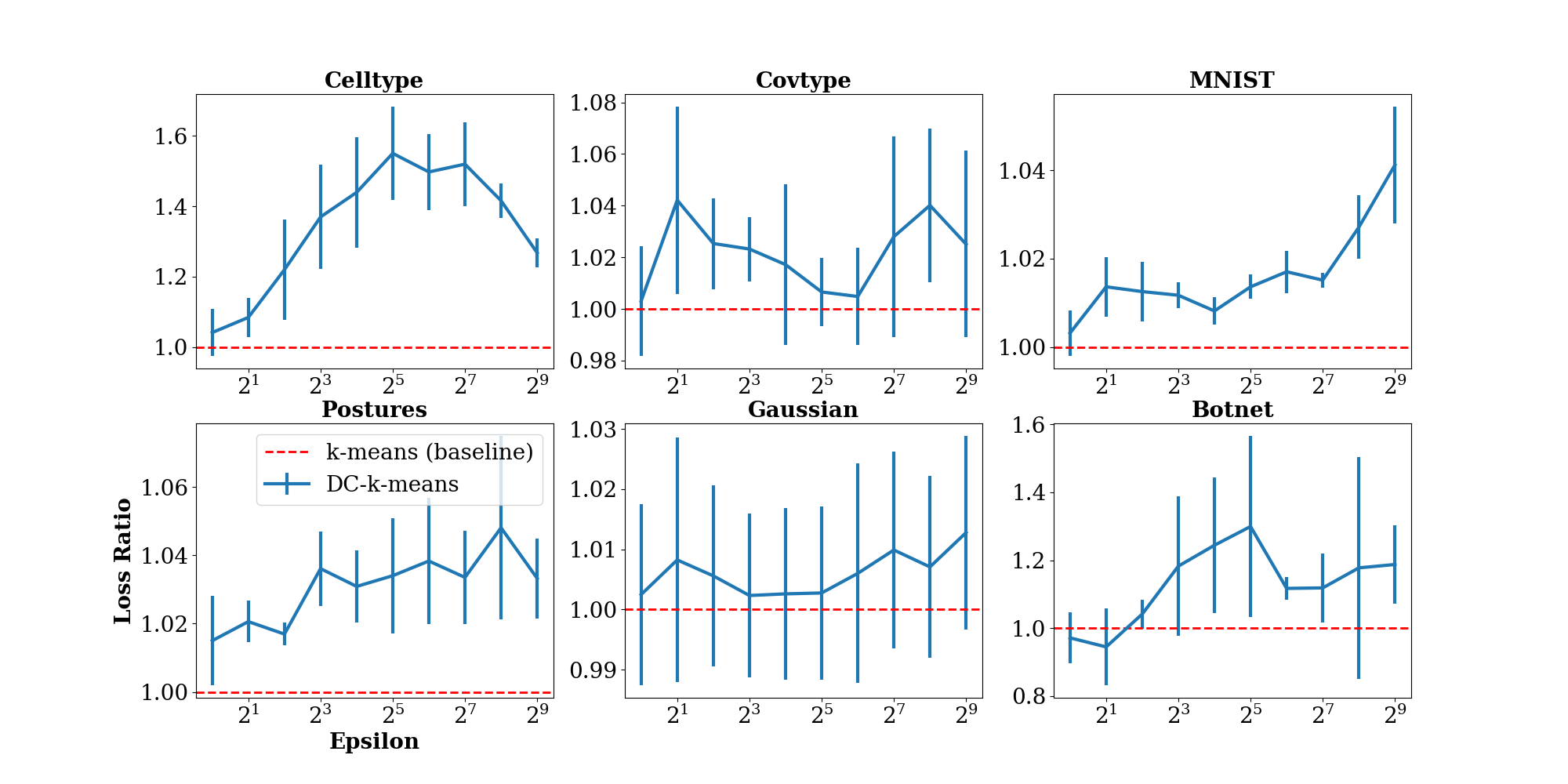}
    \caption{Loss Ratio vs. $w$ for DC-$k$-means on 6 datasets}
  \label{fig:w_vs_acc}
\end{figure}
In Fig. 3, Q-$k$-means performance rapidly deteriorates as $\epsilon \rightarrow 1$. This is fairly expected given our theoretical analysis, and is also consistent across the six datasets.

On the other hand, in Fig. 4, we see that the relationship between $w$ and loss is far weaker. The general trend among the datasets is that performance decreases as width increases, but this is not always monotonically the case. As was mentioned in the main body, it is difficult to analyze relationship between loss and $w$ theoretically, and, for some datasets, it seems variance amongst different random seeds can dominate the impact of $w$.

\subsubsection{Effects of Quantization Granularity and Tree Width on Deletion Efficiency}

On the \texttt{Covtype} dataset, we plot the amortized runtimes on the deletion benchmark for a sweep of $\epsilon$ and $w$ for both Q-$k$-means and DC-$k$-means, respectively. As expected, the runtimes for Q-$k$-means monotonically increase as $\epsilon \rightarrow 0$. The runtimes for DC-$k$-means are minimized by some an optimal tree width at approximately 32-64 leaves. 

\begin{figure}[h!]
\centering
\begin{minipage}{.42\textwidth}
  \centering
  \includegraphics[width=0.95\textwidth]{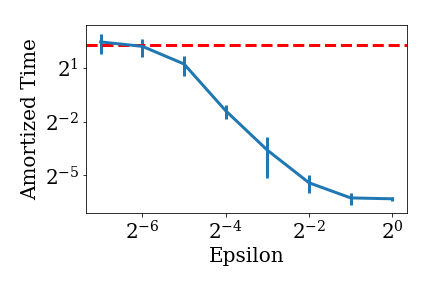}
    \caption{Amortized runtime (seconds) for Q-$k$-means as a function of quantization granularity on \texttt{Covtype}}
  \label{fig:eps_vs_rt}
\end{minipage}
\begin{minipage}{0.06\textwidth}
  \hspace{10pt}
\end{minipage}{}
\begin{minipage}{.42\textwidth}
  \centering
 \includegraphics[width=0.95\textwidth]{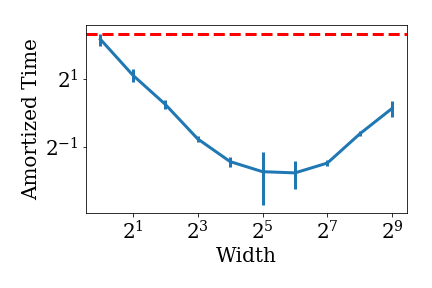}
    \caption{Amortized runtime (seconds) for DC-$k$-means as a function of tree width on \texttt{Covtype}}
  \label{fig:w_vs_rt}
\end{minipage}
\end{figure}

\section{Extended Discussion}

We include an extended discussion for relevant and interesting ideas that are unable to fit in the main body.

\subsection{Deletion Efficiency vs. Statistical Performance}

In relational databases, data is highly structured, making it easy to query and delete it. This is not the case for most machine learning models. Modern learning algorithms involve data processing that is highly complex, costly, and stochastic. This makes it difficult to efficiently quantify the effect of an individual datapoint on the entire model. Complex data processing may result in high-quality statistical learning performance, but results in models for which data deletion is inefficient, and, in the worst case, would require re-training from scratch. On the other hand, simple and structured data processing yields efficient data deletion operations (such as in relational databases) but may not boast as strong statistical performance. \emph{This is the central difficulty and trade-off engineers would face in designing deletion efficient learning systems.}

Hence, we are primarily concerned with \emph{deletion efficiency} and \emph{statistical performance} (i.e. the performance of the model in its intended learning task). In principle, these quantities can both be measured theoretically or empirically. We believe that the amortized runtime in the proposed online deletion setting is a natural and meaningful way to measure deletion efficiency. For deletion time, theoretical analysis involves finding the amortized complexity in a particular asymptotic deletion regime. In the empirical setting, we can simulate sequences of online deletion requests from real datasets and measure the amortized deletion time on wall-clocks. For statistical performance, theoretical analysis can be difficult but might often take the shape of a generalization bound or an approximation ratio. In the empirical setting, we can take the actual optimization loss or label accuracy of the model on a real dataset.

\subsection{Overparametrization, Dimensionality Reduction and Quantization}
One primary concern with quantization is that it performs poorly in the face of overparameterized models. In some situations, metric-preserving dimensionality reduction techniques \cite{johnson1984extensions,dasgupta2003elementary} could potentially be used.

\subsection{Hyperparameter Tuning}

Hyperparameter tuning is an essential part of many machine learning pipelines. From the perspective of deletion efficient learning, hyperparameter tuning presents somewhat of a conundrum. Ultimately, in scenarios in which hyperparameter tuning does indeed fall under the scope of deletion, one of the wisest solutions may be to tune on a subset of data that is unlikely to be deleted in the near future, or to pick hyperparameters via good heuristics that do not depend on specific datapoints.

\end{document}